\documentclass[sigconf]{acmart}
\AtBeginDocument{%
  }

\copyrightyear{2025}
\acmYear{2025}
\setcopyright{acmlicensed}
\acmConference[CIKM '25] {Proceedings of the 34th ACM International Conference on Information and Knowledge Management}{ November 10--14, 2025}{Seoul, Republic of Korea.}
\acmBooktitle{Proceedings of the 34th ACM International Conference on Information and Knowledge Management (CIKM '25), November 10--14, 2025, Seoul, Republic of Korea}
\acmISBN{979-8-4007-2040-6/2025/11}
\acmDOI{10.1145/XXXXXX.XXXXXX}

\usepackage{amsmath}
\usepackage{multirow}
\usepackage{color}
\usepackage{graphics}

\usepackage{xcolor, soul}
\sethlcolor{lightgray}

\usepackage{tabularx}

\definecolor{aluminum}{RGB}{153,153,153}
\definecolor{platinum}{RGB}{228,228,228}
\definecolor{bgc}{RGB}{245,245,245}
\definecolor{gallery}{RGB}{240,240,240}
\definecolor{tuatara}{RGB}{67, 67, 67}
\definecolor{flamingo}{RGB}{237, 88, 85}
\definecolor{salmon}{RGB}{242,131,107}
\definecolor{free_speech_aquamarine}{RGB}{0, 156, 114}

\definecolor{Aquamarine3}{RGB}{102, 205, 170} 
\definecolor{Goldenrod1}{RGB}{255, 193, 37} 
\definecolor{IndianRed1}{RGB}{255, 106, 106} 
\definecolor{SlateBlue1}{RGB}{131, 111, 255}

\definecolor{bb}{HTML}{95e1d3}
\definecolor{gg}{HTML}{c7ffd8}
\definecolor{yy}{HTML}{f0c38e}
\definecolor{blu}{HTML}{5ab4ba}
\definecolor{rr}{HTML}{f38181}

\definecolor{c1}{HTML}{6E85B2}
\definecolor{c2}{HTML}{368B85}
\definecolor{c3}{HTML}{C56824}
\definecolor{c5}{HTML}{916BBF}

\usepackage{subcaption}

\usepackage{algorithm}
\usepackage{algpseudocode}

\newtheorem{theorem}{Theorem}

\usepackage[most]{tcolorbox}

\usepackage{enumitem}
\usepackage{cleveref}

\begin{document}

\title{Adaptive Heterogeneous Graph Neural Networks: Bridging Heterophily and Heterogeneity}

\author{Qin Chen}
\orcid{0000-0003-1808-1585}
\affiliation{
 \institution{State Key Laboratory of General Artificial Intelligence, School of Intelligence Science and Technology, Peking University}
 \city{Beijing}
 \country{China}}

\email{chenqink@pku.edu.cn}

\author{Guojie Song}
\orcid{0000-0001-8295-2520}
\authornote{Corresponding author.}
\affiliation{
 \institution{State Key Laboratory of General Artificial Intelligence, School of Intelligence Science and Technology, Peking University}
 \city{Beijing}
 \country{China}}
\email{gjsong@pku.edu.cn}

\renewcommand{\shortauthors}{Qin Chen, Guojie Song}

\begin{abstract}
Heterogeneous graphs (HGs) are common in real-world scenarios and often exhibit heterophily. However, most existing studies focus on either heterogeneity or heterophily in isolation, overlooking the prevalence of heterophilic HGs in practical applications. Such ignorance leads to their performance degradation. In this work, we first identify two main challenges in modeling heterophily HGs:  (\romannumeral 1) varying heterophily distributions across hops and meta-paths; (\romannumeral 2) the intricate and often heterophily-driven diversity of semantic information across different meta-paths. Then, we propose the Adaptive Heterogeneous Graph Neural Network (AHGNN) to tackle these challenges. AHGNN employs a heterophily-aware convolution that accounts for heterophily distributions specific to both hops and meta-paths. It then integrates messages from diverse semantic spaces using a coarse-to-fine attention mechanism, which filters out noise and emphasizes informative signals. Experiments on seven real-world graphs and twenty baselines demonstrate the superior performance of AHGNN, particularly in high-heterophily situations. 
\end{abstract}



\maketitle

\section{Introduction} 
\label{sec:intro}
Real-world systems often exhibit complex relationships that can be effectively modeled using heterogeneous graphs (HGs) \cite{sun2012mining}, which capture a diverse range of node and edge types. While conventional Graph Neural Networks (GNNs) struggle with this diversity, specialized Heterogeneous Graph Neural Networks (HGNNs) have been developed to extract structural patterns from these networks \cite{fu2020magnn,hu2020heterogeneous,mao2023hinormer,yang2023simple}. In parallel, many real-world graphs exhibit heterophily \cite{mcpherson2001birds}, where connected nodes are frequently dissimilar, often manifesting as label inconsistency \cite{bandyopadhyay2005link}. This is in stark contrast to the homophily assumption underlying conventional GNNs \cite{mcpherson2001birds}, leading to diminished performance on heterophily graphs \cite{zhu2020beyond}. This challenge has spurred significant research, primarily focused on homogeneous graphs with single node and edge types \cite{chien2020adaptive,li2022finding,xu2023node,chen2025dagprompt}.

\begin{figure}[h]
    \centering
    \begin{subfigure}[b]{\linewidth}
        \centering
        \includegraphics[width=\linewidth]{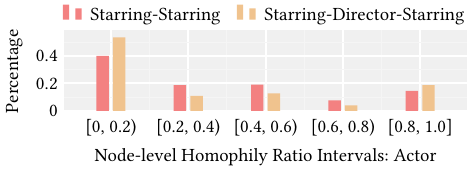}
    \end{subfigure}
    \hfill
    \begin{subfigure}[b]{\linewidth}
        \centering
        \includegraphics[width=\linewidth]{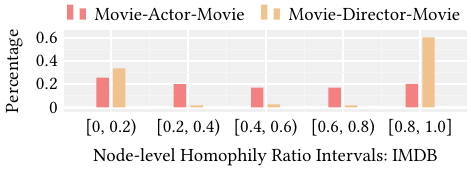}
    \end{subfigure}
    \caption{Distributions of local metapath-induced homophily ratios of Actor and IMDB. For clarity, we categorized them into five bins.}
    \label{fig:local_homophily}
\end{figure}


However, despite the widespread presence of real-world graphs that exhibit both heterogeneity and heterophily, this intersection remains underexplored. Our analysis reveals that \textbf{significant heterophily exists in heterogeneous graphs}. More importantly, from a meta-path perspective, \textbf{different meta-paths can exhibit distinct heterophily distributions}. To illustrate this, we visualize the node-level homophily ratio (see \autoref{sec:preliminary} for formal definitions) for two widely-used heterogeneous datasets: Actor \cite{guo2023homophily} and IMDB \cite{lv2021we}, as shown in \autoref{fig:local_homophily}. In the Actor dataset, we observe notable heterophily: \textit{Actors} connected through the same \textit{Movie} or \textit{Director} do not consistently share similar attributes. Interestingly, the choice of meta-path significantly affects the heterophily distribution. For instance, in IMDB, the \textit{Movie-Director-Movie} meta-path yields more extreme homophily ratios—closer to 0 or 1—compared to the \textit{Movie-Actor-Movie} path. This demonstrates that different meta-paths involve varying degrees and patterns of heterophily.

Traditional HGNNs generally omit the heterophily issues, leading to their sub-optimal performance in some real-world HGs with inherent heterophily.  Although there have been efforts to extend heterophily-oriented models to account for heterogeneity \cite{guo2023homophily}, these adaptations have generally yielded sub-optimal results due to lack of meta-path specific concerns, and sometimes it hurts the model on non-heterophily graphs. In fact, several widely used HGNNs outperform these rewired models, as evidenced in \autoref{table:main-node-classification}. This observation raises a key question: \textit{How can we effectively model heterogeneous graphs with inherent heterophily?}

This paper highlights two key challenges in mining heterophily HGs:
(\romannumeral 1) \textbf{The variation in heterophily distribution across different hops and meta-paths}. Heterophily levels can vary significantly across different hops and meta-paths (see \autoref{fig:local_homophily} for an example). A one-size-fits-all approach to model heterophily across all hops (as induced by meta-paths) and meta-paths may be sub-optimal. Instead, an adaptive strategy tailored to the specific characteristics of each hop and meta-path is necessary to effectively capture heterophily.
(\romannumeral 2) \textbf{the complex and often heterophily-influenced variation in semantic information across different meta-paths}. Semantic information from meta-paths becomes particularly intricate and multifaceted when these paths traverse heterophily connections and bridges disparate conceptual domains. Consequently, this complexity can render certain meta-paths irrelevant or even noisy for specific nodes or tasks, necessitating models where different nodes can selectively prioritize or filter meta-path information according to their unique context, rather than treating all meta-paths uniformly.

In response to these challenges, we propose the Adaptive Heterogeneous Graph Neural Network (AHGNN), tailored for both homophily and heterophily HGs. AHGNN comprises two components:  (\romannumeral 1) An Adaptive Heterogeneous Convolution module for hop and meta-path specific heterophily-aware convolution. (\romannumeral 2) A Coarse-to-Fine Semantic Fusion module for selective semantic information integration from different meta-paths. We conduct a comprehensive evaluation on twenty baselines and seven real-world datasets. AHGNN consistently achieves state-of-the-art results. Its performance boost is notably marked in strong heterophily scenarios, showing up to a 4.32\% increase in Micro-F1 score. Our few-shot and synthetic experiments further support this finding. Additionally, AHGNN proves to be computationally efficient in our efficiency analysis. We also provide a theoretical analysis of the Adaptive Heterogeneous Convolution.

Generally, our contribution is summarized as follows: 
\begin{itemize}
    \item We identify the two unique challenges in modeling heterogeneous graphs with heterophily. 
    \item We propose a novel model, AHGNN, which handles homophily and heterophily HGs adaptively while maintaining computational efficiency.
    \item We conducted extensive experiments where AHGNN achieves state-of-the-art performance with up to 4.32\% increase in Micro-F1 score. The performance improvement is especially significant on graphs with stronger heterophily. Few-shot and synthetic experiments also corroborate this observation.
\end{itemize}

\section{Preliminary}
\label{sec:preliminary}


\begin{definition}[Meta-path Induced Sub-graph]
\label{def:mp_subgraph}
Given a meta-path $\mathcal P = \mathcal{T}_1 \mathcal{T}_2 \cdots \mathcal{T}_L$, with adjacency matrices $\mathbf A^{\mathcal{T}_i \mathcal{T}_j}$ representing connections between types $\mathcal{T}_i$ and $\mathcal{T}_j$, the sub-graph $\mathcal G_{\mathcal P}$ induced by $\mathcal P$ is recursively defined as: 
$\mathbf A^{\mathcal P} = \mathbf A^{\mathcal T_1 \cdots \mathcal T_L} = \mathbf A^{\mathcal T_1 \mathcal T_2} \mathbf A^{\mathcal T_2 \cdots \mathcal T_L},$
where $\mathbf A^{\mathcal P}$ is its adjacency matrix. The node set of $\mathcal G_{\mathcal P}$ includes nodes $v_i$ and $v_j$ such that $\phi(v_i) = \mathcal T_1, \phi(v_j)= \mathcal T_L$, and there is a path from $v_i$ to $v_j$ conforming to $\mathcal P$.
\end{definition}

\begin{definition}[Homophily Ratio for Homogeneous Graphs]
\label{def:homo}
Given a homogeneous graph $\mathcal G = (\mathcal V, \mathcal E)$, the \textbf{global} homophily ratio \cite{gong2024towards,zhu2020beyond} of $\mathcal G$ is defined via 
\begin{equation}
    h = \frac{| \{(v_i, v_j) : (v_i, v_j) \in \mathcal E \wedge \boldsymbol y_i = \boldsymbol y_j \}|} {|\mathcal E|},
\end{equation}
where $\boldsymbol y_i$ denotes the label of node $v_i \in \mathcal V$.
\end{definition}

\begin{definition}[Meta-path Induced Homophily Ratio]
\label{def:mp_homo}
Given a meta-path $\mathcal P = \mathcal{T}_1 \mathcal{T}_2 \cdots \mathcal{T}_L$, and assume that $\mathcal{T}_1 = \mathcal{T}_L$. Therefore, the induced graph $\mathcal G_{\mathcal P}$ is a homogeneous graph. The homophily ratio of HG is defined via meta-path-based connection, which is a natural transition from the homophily ratio in homogeneous graphs \cite{zhu2020beyond}.
The \textbf{global} homophily ratio of $\mathcal G_{\mathcal P}$ is as 
\begin{equation}
h(\mathcal{G}_{\mathcal P})=\frac{\left|\left\{\left(v_i, v_j\right):\left(v_i, v_j\right) \in \mathcal{E}_{\mathcal P} \wedge \boldsymbol{y}_i=\boldsymbol{y}_j\right\}\right|}{|\mathcal{E}_{\mathcal P}|},
\end{equation}
where $\boldsymbol{y}_i$ denotes the label of node $v_i$. We further define the \textbf{local} homophily ratio of $\mathcal G_{\mathcal P}$ for node $v_i$ as 
\begin{equation}
h_i(\mathcal{G}_{\mathcal P})=\frac{\left.\mid\left\{\left(v_i, v_j\right): v_j \in \mathcal N_{\mathcal P, i} \wedge \boldsymbol{y}_i=\boldsymbol{y}_j\right)\right\} \mid}{\left|\mathcal N_{\mathcal P, i}\right|},
\end{equation}
where $\mathcal N_{\mathcal P, i}$ denotes the neighborhood of $v_i$ in $\mathcal G_{\mathcal P}$.
$h(\mathcal{G}_{\mathcal P})$ describes the general similarity between node linked by a certain meta-path $\mathcal P$, and $h_i(\mathcal{G}_{\mathcal P})$ is a local version only considering node $v_i$'s neighbors. A $h(\mathcal{G}_{\mathcal P})$ up to 1 indicates a strong similarity between nodes with edges, while a $h(\mathcal{G}_{\mathcal P})$ down to 0 means that there are hardly no similarities at all.

\end{definition} 

\begin{definition}[Homophily Ratio for Heterogeneous Graphs]
\label{def:h_heter}
Given a set of meta-paths $\mathcal M$, including all paths up to length $L$ starting with target node type $\mathcal T_\text{target}$, we define the graph-level homophily ratio as: 
\begin{equation}
    h = \operatorname{Mean}\left(h(\mathcal{G}_{\mathcal P})\mid \mathcal P \in \mathcal M, \mathcal P \text{ ends with } \mathcal T_\text{target} \right).
\end{equation}
$h$ represents the overall homophily level of a heterogeneous graph considering all its meta-paths. In this paper, the homophily ratio for all graphs are computed with $L=4$.
\end{definition}

\section{Model}
\label{sec:model}

\begin{figure*}
    \centering
    \includegraphics[width=\linewidth]{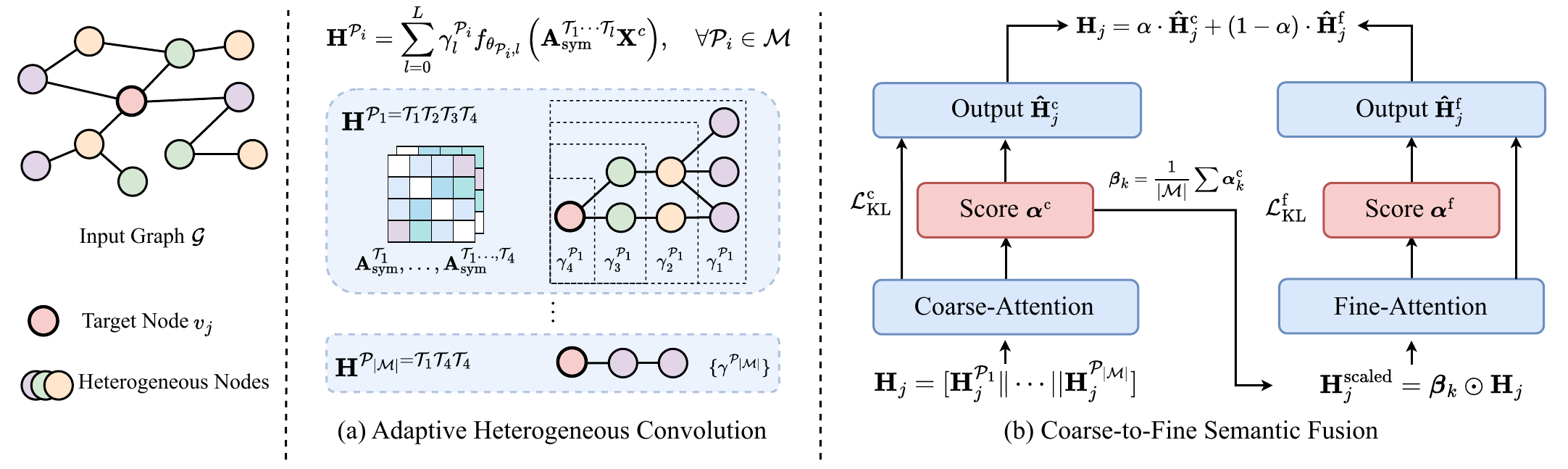}
    \caption{Framework of AHGNN. It first conducts the  Adaptive Heterogeneous Convolution per each meta-path for a path-specific embedding. Then for each node, the corresponding embeddings are processed with a Coarse-to-Fine Semantic Fusion module, to filter out irrelevant and noisy embeddings, and prioritize informative ones. }
    \label{fig:framework}
\end{figure*}

In this section, we present the proposed Adaptive Heterogeneous Graph Neural Network (AHGNN). The framework is structured into two key stages: (\romannumeral 1) Adaptive heterogeneous graph convolution, and (\romannumeral 2) Coarse-to-Fine Semantic fusion. An overview of the AHGNN framework is shown in \autoref{fig:framework}.

\subsection{Adaptive Heterogeneous Convolution}
\label{sec:adaptive_hconv}

In this subsection, we propose the Adaptive Heterogeneous Convolution, which (\romannumeral 1) propagates heterogeneous messages across the meta-paths in an efficient manner; (\romannumeral 2) accommodates heterophily data, aligning with the inherent characteristics of graphs where heterophily distributions can differ across hop-neighborhoods and among meta-paths.

In the field of homogeneous heterophily data mining, various studies \cite{zhu2020beyond,chien2020adaptive} have introduced methods like (\romannumeral 1) hop-separated neighborhood aggregation and (\romannumeral 2) hop-separated transformation. The Adaptive Heterogeneous Convolution draws inspiration from these techniques, adapting the concept from \textbf{homogeneous} to \textbf{heterogeneous} contexts. It shifts from direct edge-linking to meta-path-based edge-linking. Embeddings along a meta-path of length $L$ are viewed as separate messages ranging from $0$-hop (central) to $L-1$ hop. Heterophily distributions can differ across hops \cite{zhu2020beyond}. For instance, a node's immediate neighborhood might predominantly exhibit heterophily, whereas its second-order neighborhood could be homophily-dominant. In heterogeneous scenarios, this variance becomes even more pronounced across different meta-paths. For instance, various meta-paths might exhibit distinct homophily distributions, as illustrated in \autoref{fig:local_homophily}. This variability underscores the need for the adaptive modeling of embeddings within each meta-path (for multi-hop) and across multiple meta-paths (for multi-meta-path).

\subsubsection{Heterogeneous Message Propagation.} As outlined in \autoref{sec:preliminary}, for a given meta-path $\mathcal P_i$, the adjacency matrix of the sub-graph it induces is denoted as $\mathbf A^{\mathcal P_i}$. This matrix, $\mathbf A^{\mathcal P_i} =  \mathbf A^{\mathcal T_1 \cdots \mathcal T_L} $, can be progressively calculated from shorter meta-paths. Node features, represented by \(\mathbf X\), are then propagated along the meta-path. Importantly, the message propagation for any meta-path is performed once and only once in the SGC style \cite{wu2019simplifying}, \textbf{before training}. We enumerate all possible meta-paths starting from target node type $\mathcal T_1$ up to length $L_1$, forming the set $\mathcal M$. For each meta-path $\mathcal P_i$ in $\mathcal M$, we calculate the heterogeneous message propagation:
\begin{equation}
    \label{eq:heterogeneous_mp_1}
    \mathbf S^{\mathcal P_i} =\left\{ \mathbf{\hat A}_\text{sym} ^{\mathcal T_1\cdots \mathcal T_l}\mathbf X \mid l= 1, \dots, L \right\},
\end{equation}
where $\mathbf{\hat A}_\text{sym} = \mathbf D^{-\frac 12}\mathbf A\mathbf D^{-\frac 12}$ is the normalized symmetric version of $\mathbf{A}$, $\mathbf D$ denotes the degree matrix, and $\mathbf S^{\mathcal P_i}$ is the ordered set of embeddings, containing messages from $\mathcal T_1 \cdots \mathcal T_L$ down to $\mathcal T_1$, \textit{in descending order} of meta-path lengths. The pre-calculation enhances the efficiency of the AHGNN. For further details on its efficiency,  please refer to \autoref{sec:complexity} and \autoref{sec:efficiency}.

\subsubsection{Adaptive Heterogeneous Convolution.} The Adaptive Heterogeneous Convolution is formulated based on the heterogeneous messages previously calculated. For a given meta-path $\mathcal P_i \in \mathcal M$ and its associated embedding set  $\mathbf S^{\mathcal P_i}$, the convolution is defined as:
\begin{equation}
    \label{eq:heterogeneous_mp_2}
   \mathbf H^{\mathcal P_i} = \sum_{l=0}^L \gamma_l^{\mathcal P_{i}} f_{\theta_{\mathcal P_i, l}} \left( \mathbf S^{\mathcal P_i}_l \right),
\end{equation} 
where $\mathbf S^{\mathcal P_i}_l$ represents the $l-$th embedding in the ordered set $\mathbf S^{\mathcal P_i}$. The function $f_{\theta_{\mathcal P_i, l}}$ is a linear layer that projects heterogeneous messages of different dimensions into a unified latent space $\mathbb R^d$. For meta-paths with common components, the weights of $f_{\theta_{\mathcal P_i, l}}$ are shared, reducing computational load and preserving correlation among shared meta-path components. For example, in meta-path \textit{author-paper-author} and \textit{author-paper-conference}, the projection for \textit{author} and \textit{author-paper} is shared across the model. Drawing inspiration from \cite{gasteiger2018predict}, the learnable parameters $\{\gamma_l^{\mathcal P_{i}}\}$ are initialized based on the formula $\gamma_l^{\mathcal P_{i}} = \alpha(1-\alpha)^l$ for $l < L-1$, and $\gamma_{L-1}^{\mathcal P_{i}} = (1-\alpha)^{L-1}$, where $\alpha\in (0,1)$ is a hyper-parameter. Initially, as $\alpha \in (0, 1)$, more weight is assigned to longer meta-path components while they remain learnable, allowing the model to adaptively adjust the correlations between different hops. 

These parameters adaptively characterize \textbf{hop-level} relations within each meta-path. AHGNN employs distinct $\{\gamma_l\}$ sets for different meta-paths, accommodating the varied homophily distributions \textbf{across meta-paths}. The Adaptive Heterogeneous Convolution is initialized to approximate polynomial graph filters while allowing flexibility to accommodate other graph patterns, see \autoref{sec:theory} for details.

\subsection{Coarse-to-Fine Semantic Fusion}

In heterogeneous graphs (HGs), multiple meta-paths represent diverse semantic spaces. The semantic information within these spaces can vary significantly, and different nodes may prioritize certain spaces over others. This phenomenon is more profound in heterophily HGs, where certain spaces may be helpful for some nodes while being less influential, or even harmful too others. This raises a critical question in modeling heterophily HGs: \textit{Among diverse meta-paths, how to selectively choose the most informative ones that are representative and beneficial for downstream tasks?}

To address this challenge, we propose a two-level attention mechanism. The first level applies \textbf{coarse-grained attention} across all available meta-paths to identify informative relationships within the heterogeneous graph. Based on the evaluated importance, the meta-paths are re-weighted, and \textbf{fine-grained attention} is then applied. This approach eliminates the need for manual meta-path selection by learning which meta-paths are most relevant to the task. Furthermore, a \textbf{KL divergence loss} is employed to encourage specialization among attention heads at each level, ensuring they capture diverse and complementary graph semantics for enhanced expressiveness.

\subsubsection{Coarse-Grained Attention} 
The first step applies \textbf{coarse-grained attention} to assess the relevance of each meta-path to the task. Formally, for a graph $\mathcal G$  with a meta-path set $\mathcal M$ of size $S=|\mathcal M|$, the Adaptive Heterogeneous Convolution module generates embeddings $\{ \mathbf H^{\mathcal P_i} \in \mathbb R^{N\times d} \mid \mathcal P_i \in \mathcal M\}$. These embeddings are vertically stacked to form an input sequence $\mathbf H \in \mathbb R^{N\times S\times d}$, considering all components in $\mathcal M$. For target node $v_j$, the Coarse-Grained Attention is formulated as \footnote{For clarity, we describe the procedure using single-head attention, although it is implemented with multi-head attention in practice.}:
\begin{equation}
    \label{eq:gated_tfm}
    \begin{aligned} \mathbf Q &= \mathbf H_j\mathbf W^{Q},   \quad \mathbf K= \mathbf H_j \mathbf W^{K}, \quad \mathbf V = \mathbf H_j \mathbf W^{V} \\ 
    \boldsymbol \alpha^\text{c} &= \text{softmax}\left(\frac{\mathbf Q\mathbf K^\top}{\sqrt{d}}\right), \quad
    \mathbf{\hat H}_j^\text{c}= (\boldsymbol{\alpha^\text{c}} \mathbf V)\mathbf W^O, 
    \end{aligned}
\end{equation}
where $\mathbf W^Q, \mathbf W^K, \mathbf W^V, \mathbf W^O$ are projection matrices.

\subsubsection{Soft Meta-path Selection with Influence Factors}
We then utilize \(\boldsymbol \alpha^\text{c} \) to filer out the irrelevant or harmful meta-paths. Instead of directly omitting certain candidates which may potentially discard useful but subtle signals, we introduce a soft meta-path weighting technique through a \textit{influence factor} \(\boldsymbol \beta \in \mathbb{R}^S\), which scales the contribution of each meta-path embedding based on its importance: $ \boldsymbol \beta_k = \frac{1}{|\mathcal M|}\sum\boldsymbol\alpha_k^\text{c}$, where  \(\boldsymbol \alpha_k^\text{c} \) represents the attention weight for the \( k \)-th meta-path based on how informative it is for the current node. $\boldsymbol \beta$ softly modulating the strength of each meta-path by reweighting them as:
\begin{equation}
    \mathbf{H}_j^{\text{scaled}}[k] = \boldsymbol\beta_k \cdot \mathbf{H}_j[k],
\end{equation}
where \(\mathbf{H}_j[k] \in \mathbb{R}^{d}\) denotes the embedding of node \(v_j\) under meta-path \(\mathcal{P}_k\).

\subsubsection{Fine-Grained Attention on Modulated Meta-paths}
We then apply \textbf{fine-grained attention} to the scaled meta-path embeddings to refine node representations, attending more carefully to interactions between selected meta-paths and task-relevant local semantics. This is formalized as:

\begin{equation}
    \begin{aligned}
    \mathbf{Q}^{\text{f}} &= \mathbf{H}_j^{\text{scaled}} \mathbf{W}^{Q}, \; 
    \mathbf{K}_k^{\text{f}} = \mathbf{H}_j^{\text{scaled}} \mathbf{W}^{K}, \; 
    \mathbf{V}^{\text{f}} = \mathbf{H}_j^{\text{scaled}} \mathbf{W}^{V}, \\
    \boldsymbol \alpha^{\text{f}} &= \text{softmax} \left( \frac{\mathbf{Q}^{\text{f}} {\mathbf{K}^{\text{f}}}^\top}{\sqrt{d}} \right), \quad
    \mathbf{\hat H}_j^\text{f}= (\boldsymbol{\alpha^\text{f}} \mathbf V)\mathbf W^O,
    \end{aligned}
\end{equation}
where $\mathbf W^Q, \mathbf W^K, \mathbf W^V, \mathbf W^O$ are projection matrices.

\subsubsection{Final Fusion and KL Divergence Regularization}

The final node representation is obtained by combining coarse and fine-level outputs via a learnable weighted sum:

\begin{equation}
\label{eq:final}
    \mathbf{H}_j = \alpha \cdot \mathbf{\hat{H}}_j^\text{c} + (1 - \alpha) \cdot \mathbf{\hat{H}}_j^\text{f},
\end{equation}
where \(\alpha \in [0, 1]\) is a learnable scalar parameter initialized to 0.5.



This framework enables the model to automatically identify, weigh, and refine semantically meaningful meta-paths, making it powerful for challenging heterophily settings in heterogeneous graphs.

\subsection{Implementation Details}
This subsection outlines additional implementation details of AHGNN to enhance its performance.

\textbf{Label Propagation.} The label propagation process closely mirrors the Heterogeneous Message Propagation described in \autoref{sec:adaptive_hconv}. It is pre-computed before the training for all meta-paths starting from $\mathcal T_1$ up to a maximum length \(L_2\) in the graph:
\begin{equation}
\begin{aligned}
    \label{eq:heterogeneous_mp_label}
    \hat{\mathbf S}^{\mathcal P_i} &=\left\{ \mathbf{\hat A}_\text{sym} ^{\mathcal T_1\cdots \mathcal T_l}\mathbf Y^{\text{train}, c} \mid l= 1, \dots, L \right\} \\ 
    \hat{\mathbf H}^{\mathcal P_i} &= \sum_{l=0}^L \gamma_l f_{\theta_{\mathcal P_i, l}}' \left( \hat{\mathbf S}^{\mathcal P_i}_l \right),
\end{aligned}
\end{equation}
where $f_{\theta_{\mathcal P_i, l}}$ is a linear layer. The key difference lies in the propagation of one-hot labels $\mathbf Y^{\text{train}, c}$ from the training set rather than node features. These generated embeddings are then appended to the Gated Transformer's input sequence. For instance, if there are $S$ meta-paths with messages and $S'$ meta-paths with label-messages, then $\mathbf H_j$ in \autoref{eq:gated_tfm} will of shape $\mathbb R^{(S+S')\times d}$.

\textbf{$L_2$-Normalization.} In line with the approach in \cite{lv2021we}, we employ an $L_2$-Normalization technique for the final embedding $\mathbf{\tilde H}_j$ in \autoref{eq:final}: $    \mathbf{\tilde H}_j = \frac{\mathbf{\tilde H}_j}{\| \mathbf{\tilde H}_j \|_2}$.

\subsection{Complexity Analysis}
\label{sec:complexity}
The computational complexity of the proposed AHGNN model is near-linear. For \textbf{Adaptive Heterogeneous Convolution}, the Heterogeneous Message Propagation calculated before the training has a complexity of \(\mathcal{O}(LEF)\) per meta-path, where \(L\) is the meta-path length, \(E\) is the average number of edges in subgraphs, and \(F\) is the feature dimension. During convolution, the complexity is \(\mathcal{O}(LNdF)\) per meta-path, with \(N\) as the number of nodes and \(d\) the hidden dimension. Label propagation adds \(\mathcal{O}(LNdC)\), where \(C\) is the number of classes. For \textbf{Coarse-to-fine Semantic Fusion}, each node processes \(S\) meta-paths, yielding input tokens \(\mathbf{H}_i \in \mathbb{R}^{S \times d}\). Fusion has a per-node complexity of \(\mathcal{O}(S^2d)\) and a total graph complexity of \(\mathcal{O}(NS^2d)\). Generally, the complexity of AHGNN is approximately linear to $N$.

\section{Theoretical Analysis}
\label{sec:theory}
Here, we present a theoretical analysis of Adaptive Heterogeneous Convolution from the perspective of graph filtering. We establish the connection between the proposed Adaptive Heterogeneous Convolution and polynomial graph filters. To simplify our approach, we initially focus on homogeneous graphs.

\textbf{Notations.} Consider a homogeneous, connected graph $\mathcal{G}$ with $N$ nodes, represented by its adjacency matrix $\mathbf{A}$ and degree matrix $\mathbf{D}$. The normalized adjacency matrix is given by $\mathbf{\hat{A}}_\text{sym} = \mathbf{D}^{-\frac{1}{2}} \mathbf{A} \mathbf{D}^{-\frac{1}{2}}$, and its eigenvalue decomposition can be expressed as $\mathbf{\hat{A}}_\text{sym} = \mathbf{U} \mathbf{\Lambda} \mathbf{U}^\top$. Let $\lambda_0 \geq \lambda_1 \geq \cdots \geq \lambda_{N-1}$ denote the eigenvalues of $\mathbf{\hat{A}}_\text{sym}$. Referring to $\mathbf{\hat{A}}_\text{sym}^{\mathcal{T}_1 \cdots \mathcal{T}l}$ in \autoref{eq:heterogeneous_mp_1}, which can be rewritten as $\mathbf{\hat{A}}_\text{sym}^l$ in homogeneous settings, the corresponding polynomial graph filters for $l = 0, 1, \cdots, L$ are given by $ \sum_{l=0}^{L}\gamma_l \mathbf{\hat A}_\text{sym} = \mathbf U \beta_{\gamma, L}(\mathbf \Lambda) \mathbf U ^\top,$
where $\beta_{\gamma, L}(\mathbf{\Lambda})$ is a diagonal matrix with $\beta_{\gamma, L}(\mathbf{\Lambda})_{i,i} = \beta_{\gamma, L}(\lambda_i) = \sum_{l=0}^{L} \gamma_l \lambda_i^l.$

We begin the analysis by establishing the following lemma:

\begin{lemma}
\label{lemma:1}
For a connected graph, the largest eigenvalue \(\lambda_0\) of the normalized adjacency matrix is equal to 1, and all other eigenvalues \(\lambda_i\) satisfy \(\lambda_i < 1\) for \(i = 1, 2, \dots, N-1\).
\end{lemma}

\begin{proof}
The normalized adjacency matrix \(\mathbf{\hat{A}}_\text{sym}\) is symmetric, constructed by normalizing the adjacency matrix with the degrees of the nodes. This normalization ensures that the row and column sums are balanced, leading to the largest eigenvalue being 1. Since \(\lambda_0 \geq \lambda_i\) for \(i = 1, 2, \dots, N-1\), if there were another eigenvalue equal to 1, say \(\lambda_1 = 1\), it would imply the existence of a corresponding eigenvector. However, this would contradict the Perron-Frobenius theorem \cite{pillai2005perron}, which asserts that the largest eigenvalue of a non-negative matrix associated with a connected graph is unique. Thus, all other eigenvalues \(\lambda_i\) satisfy \(\lambda_i < 1\) for \(i = 1, 2, \dots, N-1\). 
\end{proof}

Then, the following holds:

\begin{theorem}
\label{theorem:1}
By setting $\gamma_l > 0$ (via setting $\alpha > 0$) for $l = 0, 1, \cdots, L$ and $\sum_{l=0}^L \gamma_l = 1$, if there exists $l_i \geq 0, i \ne 0$, then the Adaptive Heterogeneous Convolution $\beta_{\gamma, L}$ is initialized to a low-pass graph filter with $\left|\frac{\beta_{\gamma, L}(\lambda_i)}{\beta_{\gamma, L}(\lambda_0)} \right| < 1$ strictly for any $i = 1, 2, \cdots, N-1$.
\end{theorem}

\begin{proof}
Given the definition of \(\beta_{\gamma, L}\) and the setting of \(\gamma_l\), we have:
\begin{equation}
    \begin{aligned}
        \beta_{\gamma, L}(\lambda_0) &= \sum_{l=0}^{L}\gamma_l \lambda^l = \sum_{l=0}^{L}\gamma_l = \sum_{l=0}^{L-1}\alpha(1-\alpha)^l + (1-\alpha)^L = 1.
    \end{aligned}
\end{equation}
Since $\gamma_k > 0$, and following Lemma \autoref{lemma:1}, we have:
\begin{equation}
    \begin{aligned}
    \left|\beta_{\gamma, L}(\lambda_i) \right| = \left|\sum_{l=0}^{L}\gamma_l \lambda_i^l \right| \leq \sum_{l=0}^{L}\gamma_l \left |\lambda_i \right|^l 
     \le \sum_{l=0}^{L}\gamma_l 1^l = \sum_{l=0}^{L}\gamma_l =  1.
    \end{aligned}
\end{equation}

It is important to note that \(\sum_{l=0}^{L} \gamma_l \left|\lambda_i \right|^l = \sum_{l=0}^{L} \gamma_l \cdot 1^l\) cannot be achieved, as we assume the existence of at least one \(l_i \geq 0\) with \(i \ne 0\) and \(\lambda_i < 1\), as established in \Cref{lemma:1}. Therefore, the following inequality holds: $ \left|\frac{\beta_{\gamma, L}(\lambda_i)}{\beta_{\gamma, L}(\lambda_0)} \right| < 1.$
\end{proof}

The Adaptive Heterogeneous Convolution is initially designed to approximate a low-pass filter, which facilitates smooth learning of the graph from the outset. We make the coefficients $\gamma_l$ learnable, enabling the model to adaptively adjust the relationships between different hops (or meta-path-based hops). This approach effectively addresses the challenge posed by complex real-world heterogeneous graphs, where different meta-paths often exhibit distinct and unique intra-hop distributions and relationships (as illustrated in \autoref{fig:local_homophily}).

Fixing $\gamma_l$ to strictly follow a low-pass filter can be sub-optimal, whereas allowing meta-path-specific $\gamma_l$ provides AHGNN with the flexibility to effectively handle multiple informative meta-paths. Additionally, this learnable coefficient design allows the model to accommodate potential high-frequency signals within the graphs that are common in heterophily graphs \cite{zhu2020beyond}.

\section{Experiments}
\label{sec:exp}

\subsection{Datasets}
\label{sec:datasets}

\begin{table}[htbp]
    \centering
    \caption{Statistics of real-world datasets. $F_\text{target}$ is the feature dimension of target nodes. $h$ for graph-level homophily ratio defined in \Cref{def:h_heter}.}
    \label{table:data-statictics}
    \resizebox{.999\linewidth}{!}{
        \begin{tabular}{lrrrrrrrrrr}
            \toprule
            & \#Nodes & \#Edges & \#Class & $F_\text{target}$ & $h$ \\
            \midrule
            DBLP & 26108 & 239566 & 4 & 334 & 0.81 \\
            IMDB & 11616 & 34212 & 3 & 3066 & 0.59 \\
            ACM & 9040 & 547814 & 4 & 1902 & 0.88 \\
            FB-American & 9473 & 495790 & 3 & 6386 & 0.53 \\
            FB-MIT & 9274 & 561700 & 3 & 6440 & 0.49 \\
            Actor & 16255 & 72425 & 7 & 5362 & 0.29 \\
            Ogbn-mag & 1939743 & 42182144 & 349 & 128 & 0.34 \\
            \bottomrule
        \end{tabular}
    }
\end{table}

We evaluate AHGNN with seven real-world heterogeneous graphs with various fields, scales and homophily ratios. ACM and DBLP \cite{lv2021we} are a homophily-based citation network. IMDB \cite{lv2021we}, represents a heterophily database of online movies and TV shows. FB-American \cite{guo2023homophily} and FB-MIT \cite{traud2012social}, part of the FB100, detail Facebook users in American universities. They are also heterophily datasets. Actor \cite{tang2009social} is a strong heterophily dataset about actors, directors, and writers based on Wikipedia pages. Lastly, Ogbn-mag \cite{hu2021ogb} is a large-scale heterogeneous citation network. For all datasets, the target node type is selected according to the settings in their respective original papers.

\subsection{Involved Baselines}
To comprehensively evaluate the performance of AHGNN, we conduct experiments across four categories of models: (\romannumeral 1) \textbf{Homogeneous GNNs}: Standard models designed for homogeneous graphs, including GCN~\cite{kipf2016semi} and GAT~\cite{velivckovic2017graph}. (\romannumeral 2) \textbf{Traditional Heterogeneous GNNs}: Models specifically developed for heterogeneous graphs, such as HetGNN~\cite{zhang2019heterogeneous}, HGT~\cite{hu2020heterogeneous}, MAGNN~\cite{fu2020magnn}, SHGN~\cite{lv2021we}, SeHGNN~\cite{yang2023simple}, HINormer~\cite{mao2023hinormer}, LSMPS~\cite{li2023long}, and Seq-HGNN~\cite{du2023seq}. (\romannumeral 3) \textbf{Heterophily-aware Heterogeneous GNNs}: Methods that handle both heterogeneity and heterophily, including Hetero$^2$Net~\cite{li2023hetero}, LatGRL~\cite{shen2025heterophily} and H$^2$Gformer~\cite{lin2024heterophily}. (\romannumeral 4) \textbf{GNNs Adapted for Heterogeneity}: We adapt heterophily-aware models using HDHGR techniques~\cite{guo2023homophily}, denoting the adapted versions with a "-HD" suffix. This group includes H2GCN-HD~\cite{zhu2020beyond}, GPRGNN-HD~\cite{chien2020adaptive}, GloGNN-HD~\cite{li2022finding}, ACMGNN-HD~\cite{luan2022revisiting}, and ALTGCN-HD~\cite{xu2023node}. Additionally, we equip two state-of-the-art HGNNs with HDHGR: LSMPS-HD~\cite{li2023long} and SeqHGNN-HD~\cite{du2023seq}.

    
    
    


\begin{table*}[htbp]
\centering
\caption{Results on real-world datasets presented in Macro-F1 and Micro-F1 (scaled up by 100 for clarity), including mean and standard deviation over all runs and splits. The graph-level homophily ratio of each dataset (as defined in Definition \Cref{def:h_heter} in \autoref{sec:preliminary}) is displayed in brackets ($h$). The best results are highlighted in gray. 
}
\label{table:main-node-classification}
\resizebox{.9999\linewidth}{!}{
\setlength{\tabcolsep}{3pt}
\begin{tabular}{lcccccccccccccccccccc}
    \toprule
    &  \multicolumn{2}{c}{\textbf{Actor} (0.29)} & \multicolumn{2}{c}{\textbf{FB-MIT} (0.49)} & \multicolumn{2}{c}{\textbf{FB-American} (0.53)} & \multicolumn{2}{c}{\textbf{IMDB} (0.59)} & \multicolumn{2}{c}{\textbf{DBLP} (0.81)} & \multicolumn{2}{c}{\textbf{ACM} (0.88)} \\
    \cmidrule{2-13}
     & MacroF1 & MicroF1 & MacroF1 & MicroF1 & MacroF1 & MicroF1 & MacroF1 & MicroF1 & MacroF1 & MicroF1 & MacroF1 & MicroF1 \\
    \midrule
   GCN & 54.18{\tiny $\pm$0.33} & 64.99{\tiny $\pm$0.42} & 69.03{\tiny $\pm$2.01} & 72.05{\tiny $\pm$1.83} & 68.38{\tiny $\pm$1.24} & 71.83{\tiny $\pm$0.79} & 57.88{\tiny $\pm$1.18} & 64.82{\tiny $\pm$1.24} & 90.01{\tiny $\pm$0.34} & 91.29{\tiny $\pm$0.41} & 90.77{\tiny $\pm$0.34} & 91.90{\tiny $\pm$0.33} \\
GAT & 57.83{\tiny $\pm$1.28} & 63.10{\tiny $\pm$0.66} & 69.65{\tiny $\pm$2.35} & 72.33{\tiny $\pm$1.94} & 70.86{\tiny $\pm$1.75} & 71.49{\tiny $\pm$1.05} & 58.94{\tiny $\pm$1.35} & 64.86{\tiny $\pm$1.12} & 91.89{\tiny $\pm$0.83} & 92.27{\tiny $\pm$0.20} & 90.38{\tiny $\pm$0.24} & 91.86{\tiny $\pm$0.31} \\
\midrule
HetGNN & 61.48{\tiny $\pm$2.56} & 69.01{\tiny $\pm$2.68} & 63.14{\tiny $\pm$0.43} & 69.15{\tiny $\pm$0.80} & 60.47{\tiny $\pm$1.38} & 67.08{\tiny $\pm$0.83} & 53.46{\tiny $\pm$0.87} & 60.73{\tiny $\pm$1.49} & 91.76{\tiny $\pm$0.48} & 92.33{\tiny $\pm$0.41} & 87.19{\tiny $\pm$0.35} & 87.68{\tiny $\pm$0.22} \\
HGT & 63.72{\tiny $\pm$2.18} & 69.27{\tiny $\pm$1.34} & 63.44{\tiny $\pm$0.46} & 61.98{\tiny $\pm$0.33} & 59.65{\tiny $\pm$0.52} & 62.60{\tiny $\pm$0.41} & 63.07{\tiny $\pm$1.19} & 67.20{\tiny $\pm$1.61} & 93.01{\tiny $\pm$0.24} & 93.49{\tiny $\pm$0.25} & 90.97{\tiny $\pm$0.66} & 91.32{\tiny $\pm$0.89} \\
MAGNN & 66.74{\tiny $\pm$2.84} & 73.85{\tiny $\pm$1.39} & 72.01{\tiny $\pm$2.87} & 73.63{\tiny $\pm$2.66} & 71.63{\tiny $\pm$3.28} & 72.80{\tiny $\pm$3.01} & 67.36{\tiny $\pm$2.84} & 68.18{\tiny $\pm$2.03} & 93.28{\tiny $\pm$0.51} & 93.76{\tiny $\pm$0.45} & 91.90{\tiny $\pm$0.61} & 91.95{\tiny $\pm$0.85} \\
SHGN & 66.94{\tiny $\pm$1.37} & 76.06{\tiny $\pm$0.86} & 71.89{\tiny $\pm$2.45} & 72.70{\tiny $\pm$3.01} & 71.88{\tiny $\pm$2.71} & 72.93{\tiny $\pm$2.92} & 64.31{\tiny $\pm$1.35} & 67.05{\tiny $\pm$1.31} & 94.05{\tiny $\pm$0.31} & 94.25{\tiny $\pm$0.32} & 92.58{\tiny $\pm$0.61} & 93.01{\tiny $\pm$0.51} \\
SeHGNN & 67.51{\tiny $\pm$0.83} & 76.81{\tiny $\pm$0.48} & 72.07{\tiny $\pm$2.07} & 74.49{\tiny $\pm$1.94} & 70.39{\tiny $\pm$3.12} & 73.66{\tiny $\pm$2.18} & 67.11{\tiny $\pm$1.24} & 67.93{\tiny $\pm$1.93} & 94.24{\tiny $\pm$0.53} & 94.70{\tiny $\pm$0.41} & 93.27{\tiny $\pm$0.72} & 93.31{\tiny $\pm$0.85} \\
HINormer & 67.12{\tiny $\pm$2.18} & 77.17{\tiny $\pm$2.85} & 71.23{\tiny $\pm$1.53} & 72.08{\tiny $\pm$2.06} & 71.03{\tiny $\pm$2.64} & 72.42{\tiny $\pm$2.66} & 64.09{\tiny $\pm$1.56} & 68.01{\tiny $\pm$1.96} & 94.20{\tiny $\pm$0.45} & 94.65{\tiny $\pm$0.34} & 92.66{\tiny $\pm$0.73} & 93.34{\tiny $\pm$0.85} \\
LSMPS & 68.43{\tiny $\pm$1.09} & 77.13{\tiny $\pm$0.65} & 72.18{\tiny $\pm$1.97} & 74.10{\tiny $\pm$2.31} & 71.85{\tiny $\pm$2.10} & 73.09{\tiny $\pm$1.94} & 67.60{\tiny $\pm$1.38} & 68.19{\tiny $\pm$1.85} & 94.97{\tiny $\pm$0.48} & 95.21{\tiny $\pm$0.63} & 93.98{\tiny $\pm$0.83} & 94.26{\tiny $\pm$0.53} \\
Seq-HGNN & 69.38{\tiny $\pm$1.42} & 77.41{\tiny $\pm$0.96} & 71.64{\tiny $\pm$2.18} & 73.51{\tiny $\pm$1.92} & 71.99{\tiny $\pm$2.07} & 73.01{\tiny $\pm$1.55} & 67.36{\tiny $\pm$1.42} & 68.21{\tiny $\pm$1.45} & 95.45{\tiny $\pm$0.48} & 95.63{\tiny $\pm$0.67} & 94.01{\tiny $\pm$0.53} & 94.23{\tiny $\pm$0.38} \\
\midrule
Hetero$^2$Net & 70.17{\tiny $\pm$1.38} & 77.85{\tiny $\pm$1.76} & 71.41{\tiny $\pm$1.78} & 72.43{\tiny $\pm$1.91} & 71.15{\tiny $\pm$1.84} & 72.33{\tiny $\pm$1.66} & 65.18{\tiny $\pm$0.48} & 68.16{\tiny$\pm$0.56} & 94.03{\tiny $\pm$0.35} & 94.46{\tiny $\pm$0.37} & 92.84{\tiny $\pm$0.37} & 93.31{\tiny $\pm$0.75} \\
LatGRL & 70.08{\tiny $\pm$1.20} & 77.52{\tiny $\pm$0.78} & 71.76{\tiny $\pm$1.70} & 74.56{\tiny $\pm$1.96} & 71.01{\tiny $\pm$2.93} & 72.58{\tiny $\pm$1.94} & 67.08{\tiny $\pm$1.28} & 68.38{\tiny $\pm$1.36} & 92.37{\tiny $\pm$0.28} & 94.38{\tiny $\pm$0.53} & 92.44{\tiny $\pm$0.45} & 93.54{\tiny $\pm$0.56} \\
H$^2$Gformer & 69.05{\tiny $\pm$1.60} & 78.42{\tiny $\pm$2.93} & 71.36{\tiny $\pm$1.79} & 72.27{\tiny $\pm$1.91} & 70.50{\tiny $\pm$2.48} & 72.51{\tiny $\pm$2.60} & 65.33{\tiny $\pm$1.54} & 67.76{\tiny $\pm$2.07} & 92.81{\tiny $\pm$0.45} & 93.48{\tiny $\pm$0.68} & 93.10{\tiny $\pm$0.81} & 93.41{\tiny $\pm$0.85} \\
\midrule
H2GCN-HD & 64.75{\tiny $\pm$1.31} & 74.09{\tiny $\pm$1.32} & 71.36{\tiny $\pm$3.08} & 73.78{\tiny $\pm$2.76} & 70.16{\tiny $\pm$1.21} & 72.05{\tiny $\pm$1.76} & 58.87{\tiny $\pm$1.64} & 59.39{\tiny $\pm$1.44} & 92.32{\tiny $\pm$0.43} & 92.81{\tiny $\pm$0.38} & 87.54{\tiny $\pm$0.56} & 90.44{\tiny $\pm$0.73} \\
GPRGNN-HD & 67.12{\tiny $\pm$1.08} & 75.53{\tiny $\pm$1.62} & 72.17{\tiny $\pm$1.40} & 74.17{\tiny $\pm$1.52} & 70.01{\tiny $\pm$1.23} & 71.57{\tiny $\pm$1.31} & 57.72{\tiny $\pm$0.64} & 61.03{\tiny $\pm$0.16} & 93.56{\tiny $\pm$0.53} & 94.17{\tiny $\pm$0.53} & 87.40{\tiny $\pm$0.43} & 90.29{\tiny $\pm$0.53} \\
GloGNN-HD & 69.33{\tiny $\pm$2.32} & 77.63{\tiny $\pm$1.84} & 71.35{\tiny $\pm$1.96} & 72.36{\tiny $\pm$1.95} & 71.01{\tiny $\pm$1.68} & 72.31{\tiny $\pm$2.04} & 65.14{\tiny $\pm$1.01} & 67.75{\tiny $\pm$1.52} & 94.22{\tiny $\pm$0.64} & 94.35{\tiny $\pm$0.41} & 92.05{\tiny $\pm$0.46} & 92.75{\tiny $\pm$0.34} \\
ALTGNN-HD & 69.20{\tiny $\pm$2.28} & 77.62{\tiny $\pm$1.81} & 71.18{\tiny $\pm$1.93} & 72.41{\tiny $\pm$1.90} & 71.15{\tiny $\pm$1.72} & 72.50{\tiny $\pm$2.08} & 64.78{\tiny $\pm$0.98} & 67.32{\tiny $\pm$1.47} & 93.95{\tiny $\pm$0.68} & 93.89{\tiny $\pm$0.43} & 93.11{\tiny $\pm$0.66} & 93.60{\tiny $\pm$0.31} \\
ACMGCN-HD & 69.45{\tiny $\pm$2.35} & 77.81{\tiny $\pm$1.87} & 71.42{\tiny $\pm$1.99} & 72.48{\tiny $\pm$1.93} & 71.92{\tiny $\pm$1.64} & 72.25{\tiny $\pm$2.01} & 65.25{\tiny $\pm$1.04} & 67.60{\tiny $\pm$1.50} & 94.10{\tiny $\pm$0.62} & 94.68{\tiny $\pm$0.45} & 92.98{\tiny $\pm$0.43} & 93.64{\tiny $\pm$0.29} \\

LSMPS-HD & 70.07{\tiny $\pm$1.21} & 78.60{\tiny $\pm$0.61} & 71.40{\tiny $\pm$2.01} & 72.26{\tiny $\pm$2.34} & 72.45{\tiny $\pm$2.21} & 72.88{\tiny $\pm$1.80} & 67.06{\tiny $\pm$1.36} & 68.16{\tiny $\pm$1.82} & 94.44{\tiny $\pm$0.49} & 94.97{\tiny $\pm$0.64} & 93.98{\tiny $\pm$1.05} & 93.62{\tiny $\pm$0.54}  \\

SeqHGNN-HD & 69.50{\tiny $\pm$1.31} & 76.78{\tiny $\pm$0.78} & 71.39{\tiny $\pm$1.86} & 72.10{\tiny $\pm$2.06} & 71.99{\tiny $\pm$2.53} & 72.92{\tiny $\pm$1.21} & 67.30{\tiny $\pm$1.26} & 67.69{\tiny $\pm$1.35} & 93.12{\tiny $\pm$0.03} & 93.87{\tiny $\pm$0.86} & 93.27{\tiny $\pm$0.14} & 93.81{\tiny $\pm$0.20}  \\

\midrule
AHGNN & \hl{74.89\tiny $\pm$0.96} & \hl{82.13\tiny $\pm$0.08} & \hl{73.81\tiny $\pm$2.07} & \hl{76.32\tiny $\pm$1.59} & \hl{73.75\tiny$\pm$1.42} & \hl{75.41\tiny$\pm$0.80} & \hl{69.81\tiny$\pm$1.34} & \hl{69.83\tiny $\pm$1.39} & \hl{95.84\tiny $\pm$0.37} & \hl{96.33\tiny$\pm$ 0.45} & \hl{94.45\tiny $\pm$0.48} & \hl{94.60\tiny $\pm$0.47} \\
\textit{Improvement} & \textit{4.72} & \textit{4.53} & \textit{1.63} & \textit{1.76} & \textit{1.3} & \textit{1.75} & \textit{2.21} & \textit{1.45} & \textit{0.39} & \textit{0.70} & \textit{0.44} & \textit{0.36} \\
    \bottomrule
\end{tabular}
}
\end{table*}

\subsection{Settings}
\label{sec:settings}
We evaluate all the models on the node classification task \cite{kipf2016semi}. We adopt a train/validation/test split ratio of 60\%/20\%/20\% for all the datasets. We use the Adam optimizer \cite{kingma2014adam} with a learning rate $\eta \in \{0.5, 1, 5 \}\times 10^{-3}$, a maximum weight decay of $5\times 10^{-6}$, and a maximum of 200 epochs to train AHGNN. For all models, the hidden dimensions are set to 256 for fairness (except for ogbn-mag we adopt 512). We search the hyper-parameters in the same scope mentioned above for baselines without public training scripts. For ogbn-mag, We adopt a four-stage training strategy with a maximum of 300 epochs for all models. No additional embeddings are adopted. For AHGNN, we choose $L_1$ in $\{2, 3, 4\}$, $L_2$ in $\{2, 3, 4\}$ and $\alpha$ in $\{0.25, 0.4, 0.6, 0.85\}$. We set $\lambda_1=\lambda_2=10^{-4}$ in this paper. We run the experiments with five NVIDIA RTX 4090 with 24GB GPU Memory. 

\subsection{Main Evaluation Results}

\begin{table}[htbp]
    \centering
    \small
    \caption{Results on ogbn-mag are reported Macro-F1 and Micro-F1 (scaled up by 100 for clarity) across all runs and splits. Best results are highlighted, excluding baselines without public code or facing out-of-memory issues.}
    \label{table:main-classification-ogb}
    \begin{tabular}{lccccccc}
        \toprule
        & Macro-F1 & Micro-F1 \\
        \midrule
        GCN&51.91{\tiny $\pm$0.40}&52.10{\tiny $\pm$0.55}\\
        GAT&52.48{\tiny $\pm$0.81}&54.03{\tiny $\pm$0.45}\\
        HGT&54.14{\tiny $\pm$0.52}&54.18{\tiny $\pm$0.09}\\
        SHGN&58.30{\tiny $\pm$0.49}&56.18{\tiny $\pm$0.85}\\
        SeHGNN&58.29{\tiny $\pm$0.33}&58.23{\tiny $\pm$0.65}\\
        HINormer&57.19{\tiny $\pm$0.39}&57.05{\tiny $\pm$0.84}\\
        LSMPS&58.94{\tiny $\pm$0.28}&57.62{\tiny $\pm$0.45}\\
        Seq-HGNN&58.28{\tiny $\pm$0.45}&58.36{\tiny $\pm$0.38}\\
        Hetero$^2$Net&58.51{\tiny $\pm$0.16}&58.60{\tiny $\pm$0.58}\\
        LatGCL&57.18{\tiny $\pm$0.63}&58.05{\tiny $\pm$0.40}\\
        \midrule
        AHGNN& \hl{60.42\tiny $\pm$0.29}&\hl{60.69\tiny $\pm$0.19}\\
        \bottomrule
    \end{tabular}
\end{table}

As shown in \autoref{table:main-node-classification} and \autoref{table:main-classification-ogb} \footnote{HDHGR techniques are computationally intensive, leading to Out-Of-Memory (OOM) issues on the Ogbn-mag dataset, which contains 42 million edges.}, AHGNN achieves state-of-the-art performance, particularly on heterophilous datasets. On heterophily graphs, AHGNN effectively disentangles complex heterophily distributions. For homophilous graphs, it also remains competitive. Notably, on the strongly heterophilous Actor dataset, AHGNN outperforms baselines with up to a 4.32\% improvement in Micro-F1.

Key observations are as follows:
(i) Conventional HGNNs perform well on homophilous graphs (e.g., DBLP) but show limited effectiveness on heterophilous graphs.
(ii) Heterophily-aware models, such as Hetero$^2$Net, LatGCL, H$^2$Gformer, and rewired variants (denoted by the -HD suffix), remain suboptimal on datasets with mixed homophily and heterophily (e.g., IMDB, FB-MIT, FB-American), often underperforming even compared to standard HGNNs.  
(iii) While HDHGR techniques generally enhance performance on heterophilous graphs, they may degrade performance on homophilous datasets; for example, LSMPS-HD and SeqHGNN-HD perform worse than their original versions on ACM and DBLP. In general, most baselines fail to address the two challenges outlined in \autoref{sec:intro}, resulting in suboptimal performance.

\subsection{AHGNN under Few-shot Scenarios}

\begin{table*}[htbp]
\centering
\caption{Few-shot learning results on real-world datasets presented in Macro-F1 and Micro-F1 (scaled up by 100 for clarity), including mean and standard deviation over all runs and splits. The graph-level homophily ratio of each dataset is displayed in brackets ($h$). The best results are highlighted in gray. 
}
\label{table:few-shot-node-classification}
\resizebox{.9999\linewidth}{!}{
\setlength{\tabcolsep}{3pt}
\begin{tabular}{lcccccccccccccccccc}
    \toprule
    &  \multicolumn{2}{c}{\textbf{Actor} (0.29)} & \multicolumn{2}{c}{\textbf{FB-MIT} (0.49)} & \multicolumn{2}{c}{\textbf{FB-American} (0.53)} & \multicolumn{2}{c}{\textbf{IMDB} (0.59)} & \multicolumn{2}{c}{\textbf{DBLP} (0.81)} & \multicolumn{2}{c}{\textbf{ACM} (0.88)} \\
    \cmidrule{2-13}
     & MacroF1 & MicroF1 & MacroF1 & MicroF1 & MacroF1 & MicroF1 & MacroF1 & MicroF1 & MacroF1 & MicroF1 & MacroF1 & MicroF1 \\
    \midrule
    GCN & 45.98{\tiny $\pm$1.30} & 48.03{\tiny $\pm$1.93} & 38.65{\tiny $\pm$0.95} & 41.32{\tiny $\pm$1.87} & 35.97{\tiny $\pm$1.57} & 38.78{\tiny $\pm$1.33} & 35.70{\tiny $\pm$1.68} & 36.01{\tiny $\pm$1.56} & 88.37{\tiny $\pm$0.72} & 89.64{\tiny $\pm$0.38} & 87.44{\tiny $\pm$0.98} & 88.68{\tiny $\pm$0.73} \\
GAT & 46.75{\tiny $\pm$1.57} & 47.75{\tiny $\pm$2.31} & 39.41{\tiny $\pm$1.34} & 41.74{\tiny $\pm$1.58} & 35.76{\tiny $\pm$1.34} & 38.09{\tiny $\pm$1.07} & 35.88{\tiny $\pm$1.56} & 36.21{\tiny $\pm$1.78} & 89.96{\tiny $\pm$0.45} & 90.18{\tiny $\pm$0.18} & 88.01{\tiny $\pm$0.75} & 88.64{\tiny $\pm$0.54} \\
\midrule
HetGNN & 44.91{\tiny $\pm$1.84} & 48.85{\tiny $\pm$1.03} & 37.23{\tiny $\pm$1.97} & 38.56{\tiny $\pm$1.39} & 35.46{\tiny $\pm$1.24} & 38.66{\tiny $\pm$1.63} & 35.12{\tiny $\pm$2.47} & 36.80{\tiny $\pm$3.01} & 89.31{\tiny $\pm$0.86} & 90.85{\tiny $\pm$1.02} & 86.36{\tiny $\pm$0.74} & 88.12{\tiny $\pm$1.01} \\
HGT & 46.77{\tiny $\pm$1.31} & 50.22{\tiny $\pm$1.95} & 38.64{\tiny $\pm$1.65} & 38.85{\tiny $\pm$1.09} & 35.78{\tiny $\pm$1.60} & 38.53{\tiny $\pm$1.32} & 37.65{\tiny $\pm$2.53} & 38.52{\tiny $\pm$2.80} & 91.05{\tiny $\pm$0.76} & 91.36{\tiny $\pm$0.55} & 89.30{\tiny $\pm$0.63} & 89.84{\tiny $\pm$0.70} \\
MAGNN & 45.90{\tiny $\pm$1.63} & 49.64{\tiny $\pm$1.88} & 41.81{\tiny $\pm$1.83} & 42.89{\tiny $\pm$2.04} & 37.01{\tiny $\pm$0.86} & 40.78{\tiny $\pm$1.51} & 40.12{\tiny $\pm$2.97} & 41.02{\tiny $\pm$3.01} & 91.07{\tiny $\pm$0.48} & 91.74{\tiny $\pm$0.67} & 89.45{\tiny $\pm$0.47} & 89.73{\tiny $\pm$0.74} \\
SHGN & 45.86{\tiny $\pm$1.92} & 50.63{\tiny $\pm$2.07} & 40.07{\tiny $\pm$1.29} & 42.73{\tiny $\pm$1.84} & 36.31{\tiny $\pm$1.63} & 38.96{\tiny $\pm$1.47} & 39.97{\tiny $\pm$3.08} & 40.32{\tiny $\pm$3.31} & 91.45{\tiny $\pm$0.45} & 92.01{\tiny $\pm$0.30} & 90.01{\tiny $\pm$0.45} & 90.47{\tiny $\pm$0.73} \\
SeHGNN & 46.37{\tiny $\pm$1.04} & 50.04{\tiny $\pm$2.53} & 41.51{\tiny $\pm$1.57} & 44.12{\tiny $\pm$1.78} & 36.83{\tiny $\pm$1.45} & 39.77{\tiny $\pm$1.57} & 40.98{\tiny $\pm$1.94} & 41.37{\tiny $\pm$1.53} & 92.04{\tiny $\pm$0.56} & 92.93{\tiny $\pm$0.42} & 91.02{\tiny $\pm$0.45} & 91.32{\tiny $\pm$0.47} \\
HINormer & 46.83{\tiny $\pm$1.57} & 50.97{\tiny $\pm$3.81} & 41.86{\tiny $\pm$1.36} & 43.90{\tiny $\pm$1.37} & 36.86{\tiny $\pm$1.73} & 40.01{\tiny $\pm$1.89} & 40.83{\tiny $\pm$2.10} & 41.14{\tiny $\pm$1.63} & 91.85{\tiny $\pm$0.47} & 92.77{\tiny $\pm$0.64} & 90.01{\tiny $\pm$0.66} & 90.34{\tiny $\pm$0.62} \\
LSMPS & 47.44{\tiny $\pm$1.50} & 50.65{\tiny $\pm$1.97} & 41.73{\tiny $\pm$1.97} & 44.42{\tiny $\pm$0.74} & 36.97{\tiny $\pm$1.90} & 39.64{\tiny $\pm$1.66} & 41.03{\tiny $\pm$2.04} & 41.18{\tiny $\pm$1.24} & 91.86{\tiny $\pm$0.76} & 92.61{\tiny $\pm$0.81} & 91.85{\tiny $\pm$0.45} & 92.41{\tiny $\pm$0.86} \\
SeqHGNN & 47.31{\tiny $\pm$1.74} & 50.73{\tiny $\pm$1.95} & 42.08{\tiny $\pm$1.64} & 44.51{\tiny $\pm$1.46} & 37.59{\tiny $\pm$1.47} & 41.05{\tiny $\pm$1.37} & 40.56{\tiny $\pm$2.37} & 40.92{\tiny $\pm$2.87} & 92.10{\tiny $\pm$0.33} & 93.01{\tiny $\pm$0.29} & 91.68{\tiny $\pm$0.53} & 92.54{\tiny $\pm$0.42} \\

\midrule
Hetero$^2$Net & 46.87{\tiny $\pm$2.05} & 50.56{\tiny $\pm$1.05} & 40.65{\tiny $\pm$1.38} & 43.31{\tiny $\pm$1.89} & 36.83{\tiny $\pm$1.37} & 37.83{\tiny $\pm$1.74} & 40.37{\tiny $\pm$2.48} & 40.44{\tiny $\pm$2.34} & 92.07{\tiny $\pm$0.42} & 92.32{\tiny $\pm$0.67} & 89.75{\tiny $\pm$0.53} & 90.11{\tiny $\pm$0.77} \\
LatGRL & 47.32{\tiny $\pm$1.20} & 51.06{\tiny $\pm$0.78} & 41.41{\tiny $\pm$1.70} & 44.84{\tiny $\pm$1.96} & 38.26{\tiny $\pm$2.93} & 41.31{\tiny $\pm$1.94} & 40.96{\tiny $\pm$1.28} & 41.66{\tiny $\pm$1.36} & 89.58{\tiny $\pm$0.28} & 90.64{\tiny $\pm$0.53} & 90.01{\tiny $\pm$0.08} & 90.32{\tiny $\pm$0.37} \\
H$^2$Gformer & 47.81{\tiny $\pm$1.72} & 50.69{\tiny $\pm$3.58} & 41.94{\tiny $\pm$1.13} & 43.82{\tiny $\pm$1.31} & 35.17{\tiny $\pm$1.49} & 40.68{\tiny $\pm$1.72} & 40.84{\tiny $\pm$1.93} & 41.05{\tiny $\pm$1.47} & 91.43{\tiny $\pm$0.39} & 92.58{\tiny $\pm$0.52} & 89.76{\tiny $\pm$0.87} & 90.16{\tiny $\pm$0.50} \\
\midrule
H2GCN-HD & 47.02{\tiny $\pm$1.68} & 48.51{\tiny $\pm$2.46} & 40.04{\tiny $\pm$1.90} & 42.66{\tiny $\pm$1.41} & 35.75{\tiny $\pm$2.02} & 38.40{\tiny $\pm$1.83} & 35.62{\tiny $\pm$2.97} & 36.24{\tiny $\pm$2.03} & 90.04{\tiny $\pm$0.65} & 90.75{\tiny $\pm$0.77} & 86.36{\tiny $\pm$0.53} & 87.10{\tiny $\pm$0.42} \\
LINKX-HD & 46.73{\tiny $\pm$1.34} & 47.01{\tiny $\pm$1.42} & 38.12{\tiny $\pm$1.58} & 41.46{\tiny $\pm$1.34} & 34.85{\tiny $\pm$1.87} & 38.79{\tiny $\pm$1.90} & 36.22{\tiny $\pm$2.17} & 36.94{\tiny $\pm$2.87} & 90.66{\tiny $\pm$0.45} & 91.86{\tiny $\pm$0.73} & 86.64{\tiny $\pm$0.67} & 86.86{\tiny $\pm$1.02} \\
GloGNN-HD & 47.11{\tiny $\pm$2.32} & 50.65{\tiny $\pm$1.84} & 40.51{\tiny $\pm$1.96} & 40.99{\tiny $\pm$1.95} & 37.01{\tiny $\pm$1.68} & 39.42{\tiny $\pm$2.04} & 40.07{\tiny $\pm$1.01} & 40.42{\tiny $\pm$1.52} & 91.38{\tiny $\pm$0.64} & 91.48{\tiny $\pm$0.41} & 89.46{\tiny $\pm$0.35} & 91.03{\tiny $\pm$0.45} \\
ALTGNN-HD & 47.48{\tiny $\pm$2.28} & 50.67{\tiny $\pm$1.81} & 40.74{\tiny $\pm$1.93} & 41.42{\tiny $\pm$1.90} & 37.35{\tiny $\pm$1.72} & 40.53{\tiny $\pm$2.08} & 40.64{\tiny $\pm$0.98} & 40.75{\tiny $\pm$1.47} & 90.93{\tiny $\pm$0.68} & 90.74{\tiny $\pm$0.43} & 90.74{\tiny $\pm$0.46} & 91.75{\tiny $\pm$0.37} \\
ACMGCN-HD & 47.51{\tiny $\pm$2.35} & 50.41{\tiny $\pm$1.87} & 40.91{\tiny $\pm$1.99} & 41.31{\tiny $\pm$1.93} & 37.23{\tiny $\pm$1.64} & 40.46{\tiny $\pm$2.01} & 40.32{\tiny $\pm$1.04} & 40.71{\tiny $\pm$1.50} & 91.08{\tiny $\pm$0.62} & 91.39{\tiny $\pm$0.45} & 90.63{\tiny $\pm$0.63} & 90.96{\tiny $\pm$0.46} \\

LSMPS-HD & 47.56{\tiny $\pm$1.44} & 50.25{\tiny $\pm$1.86} & 41.92{\tiny $\pm$1.84} & 44.95{\tiny $\pm$0.86} & 37.62{\tiny $\pm$2.08} & 40.41{\tiny $\pm$1.87} & 41.05{\tiny $\pm$1.82} & 40.95{\tiny $\pm$1.37} & 91.67{\tiny $\pm$0.66} & 92.56{\tiny $\pm$0.67} & 91.11{\tiny $\pm$0.50} & 91.62{\tiny $\pm$0.95}   \\

SeqHGNN-HD & 47.45{\tiny $\pm$1.57} & 50.78{\tiny $\pm$1.93} & 42.40{\tiny $\pm$1.49} & 45.68{\tiny $\pm$1.50} & 37.94{\tiny $\pm$1.28} & 41.72{\tiny $\pm$1.08} & 41.00{\tiny $\pm$2.46} & 41.42{\tiny $\pm$2.92} & 92.14{\tiny $\pm$0.48} & 92.40{\tiny $\pm$0.78} & 90.96{\tiny $\pm$0.56} & 91.45{\tiny $\pm$0.46} \\


\midrule
AHGNN & \hl{50.58\tiny $\pm$2.01} & \hl{54.52\tiny $\pm$1.70} & \hl{43.89\tiny $\pm$0.97} & \hl{45.57\tiny $\pm$1.15} & \hl{39.58\tiny $\pm$0.20} & \hl{43.81\tiny $\pm$1.43} & \hl{42.05\tiny $\pm$2.59} & \hl{42.57\tiny $\pm$2.60} & \hl{93.12\tiny $\pm$0.56} & \hl{93.48\tiny $\pm$0.56} & \hl{91.87\tiny $\pm$0.378} & \hl{92.99\tiny $\pm$0.38} \\
\textit{Improvement} & \textit{3.02} & \textit{2.99} & \textit{1.90} & \textit{1.67} & \textit{1.32} & \textit{2.05} & \textit{1.07} & \textit{1.30} & \textit{1.02} & \textit{0.47} & \textit{1.24} & \textit{1.24} \\
    \bottomrule
\end{tabular}
}
\end{table*}

We conduct a few-shot experiment, as detailed in \autoref{table:few-shot-node-classification}, where 20 nodes per class were randomly selected for the training set, and the remaining nodes were split evenly between validation and test sets. AHGNN consistently outperforms other models across all datasets, with particularly notable improvements on IMDB, FB-MIT, and Actor. On the Actor dataset, AHGNN achieves approximately a 3\% increase in Micro-F1 scores. These results further validate AHGNN's effectiveness under few-shot settings.

\subsection{AHGNN with Varying Heterophily Ratios}

We further conduct an experiment using a synthetic dataset, syn-DBLP, to evaluate the performance of AHGNN across different heterophily ratios compared to other models. Following \cite{zhu2020beyond}, by randomly assigning edges between nodes, we control the graph-level homophily of syn-DBLP to range from $0.8$ evenly to $0.1$. \autoref{fig:syn-exp} shows the results, where AHGNN generally achieves more notable performance advancement as the homophily ratio decreases. Such results are consistent with the observation in \autoref{table:data-statictics}, the stronger the heterophily, the larger the performance advancement.

\begin{figure}[htbp]
    \centering
    \begin{subfigure}[b]{\linewidth}
        \centering
        \includegraphics[width=\linewidth]{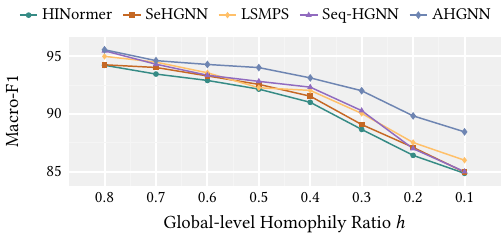}
    \end{subfigure}
    \hfill
    \begin{subfigure}[b]{\linewidth}
        \centering
        \includegraphics[width=\linewidth]{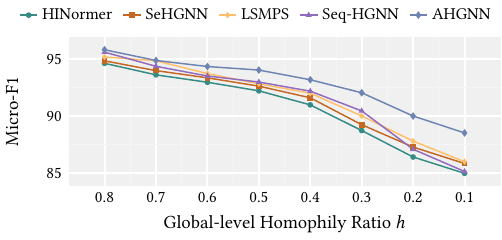}
    \end{subfigure}
    \caption{Results on syn-DBLP presented in Macro-F1 and Micro-F1 (scaled up by 100 for clarity).}
    \label{fig:syn-exp}
\end{figure}

\subsection{Ablation Study}

\begin{table}[htbp]
    \centering
    \caption{Ablation Study. "w/o" denotes without. Scores are scaled up by 100 for clarity. Here AHC stands for Adaptive Heterogeneous Convolution, C2F stands for Coarse-to-Fine Semantic Fusion, and LP for label propagation.}
    \label{table:ablation}
    \setlength{\tabcolsep}{3pt}
    \resizebox{.999\linewidth}{!}{
    \begin{tabular}{lcccccccccc}
        \toprule
        & \multicolumn{2}{c}{\textbf{IMDB} (0.59)}  & \multicolumn{2}{c}{\textbf{FB-MIT} (0.49)} & \multicolumn{2}{c}{\textbf{Actor} (0.29)} \\
         \cmidrule{2-7}
        & MacroF1 & MicroF1 & MacroF1 & MicroF1 & MacroF1 & MicroF1 \\
        \midrule
        AHGNN &69.81{\tiny $\pm$1.34}&69.83{\tiny $\pm$1.39}&73.81{\tiny $\pm$2.07}&76.32{\tiny $\pm$1.59}&74.89{\tiny $\pm$0.96}&82.13{\tiny $\pm$0.08} \\
        w/o AHC &65.01{\tiny $\pm$1.29}&66.14{\tiny $\pm$ 0.63}&71.96{\tiny $\pm$1.72}&72.84{\tiny $\pm$1.41}&69.10{\tiny $\pm$1.10}&76.48{\tiny $\pm$0.63} \\
        w/o C2F &67.80{\tiny $\pm$1.08}&68.17{\tiny $\pm$ 0.84}&70.01{\tiny $\pm$2.18}&72.18{\tiny $\pm$1.33}&71.49{\tiny $\pm$1.30}&78.35{\tiny $\pm$0.91} \\
        w/o LP &67.15{\tiny $\pm$2.47}&67.38{\tiny $\pm$ 2.48}&72.41{\tiny $\pm$1.49}&75.16{\tiny $\pm$1.72}&70.31{\tiny $\pm$0.72}&78.11{\tiny $\pm$0.29} \\
        w/o $L_2$-Norm &68.93{\tiny $\pm$1.03}&69.07{\tiny $\pm$1.06}&73.18{\tiny $\pm$1.78}&75.65{\tiny $\pm$1.93}&71.24{\tiny $\pm$0.82}&79.31{\tiny $\pm$0.51} \\ 
        \bottomrule
    \end{tabular}
    }
\end{table}

To assess the effectiveness of individual components in AHGNN, we conduct an ablation study by systematically disabling each component in isolation. For the \textit{Adaptive Heterogeneous Convolution}, we fix \(\gamma = 1\) uniformly across all meta-paths. For \textit{Coarse-to-Fine Semantic Fusion}, we replace it with a simple mean of the embeddings. The results, presented in \autoref{table:ablation}, yield several insights: (\romannumeral 1) Components contribute to performance improvements, with Adaptive Heterogeneous Convolution being the most critical. (\romannumeral 2) The impact of each component varies across datasets, with the largest gains observed on Actor, likely due to its stronger heterophily compared to IMDB and FB-MIT.

\subsection{Efficiency Analysis}

\label{sec:efficiency}

\begin{table}[htbp]
    \centering
    \small
    \caption{Efficiency analysis on Actor, including pre-calculation time (second), training time (second), and peak GPU memory consumption (MB). \texttt{-} for not-applicable. Avg Rank is the average rank on both time and memory costs. }
    \label{table:efficiency}
    \begin{tabular}{lrrrrr}
        \toprule 
        & Pre-Cal. & Training & Memory & Avg Rank \\
        \midrule
        HGT & - & 109.8 & 9415 & \#5.0\\
        SeHGNN & 7.38 & 18.36 & 7317 & \#2.0 \\
        LMSPS* & 10.04 & 79.69 & 3083 & \#3.0\\
        Seq-HGNN & - & 76.26 & 8913 & \#4.5\\
        LatGRL & - & 52.57 & 8547 & \#3.5\\
        \midrule
        AHGNN & 8.81 & 19.46 & 7103 & \#2.0 \\
        \bottomrule
    \end{tabular}
\end{table}

We conduct an efficiency analysis on AHGNN and other baselines\footnote{LMSPS's time costs were recorded for search and train stages.} on Actor. This analysis estimates the time (in seconds) and peak GPU memory usage (in MB) during pre-calculation (if applicable) and training stages. We calculate the average ranking of each model based on both time and memory considerations. Results in \autoref{table:efficiency} reveal that AHGNN is a lightweight model efficient in both time and space. On average, AHGNN ranks among the most efficient models. While consuming slightly more space than SeHGNN, AHGNN offers significant performance improvements, as shown in \autoref{table:main-node-classification}. This trade-off between efficiency and performance is considered worthwhile. 

\subsection{Parameter Analysis}

\begin{figure*}[htbp]
    \centering
    \begin{subfigure}[b]{0.32\linewidth}
        \centering
        \includegraphics[width=\linewidth]{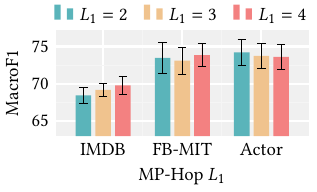}
    \end{subfigure}
    \begin{subfigure}[b]{0.32\linewidth}
        \centering
        \includegraphics[width=\linewidth]{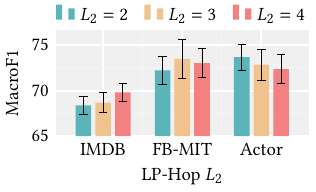}
    \end{subfigure}
    \begin{subfigure}[b]{0.34\linewidth}
        \centering
        \includegraphics[width=\linewidth]{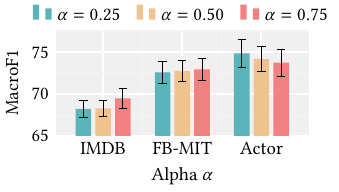}
    \end{subfigure}
    \caption{Parameter analysis on IMDB, FB-MIT and Actor.}
    \label{fig:parameter}
\end{figure*}
We conduct a parameter analysis on the maximum number of hops for heterogeneous message propagation ($L_1$), label propagation ($L_2$), and the parameter $\alpha$ in Adaptive Heterogeneous Convolution, as shown in \autoref{fig:parameter}. The results reveal that: (\romannumeral 1) optimal choices of hyper-parameters vary across different datasets due to their inherent characteristics. For instance, the Actor dataset favors shorter meta-paths, while the IMDB dataset benefits from longer ones. Additionally, the best value of $\alpha$ differs between IMDB and Actor, indicating varying levels of heterophily at different hop levels between these datasets. (\romannumeral 2) The AHGNN generally exhibits robustness to variations in hyper-parameter settings. Despite the presence of dataset-specific optimal settings, sub-optimal configurations of hyper-parameters can yield satisfying results.

\subsection{Visualization}



\begin{figure}[htbp]
    \centering
    \includegraphics[width=0.49\linewidth]{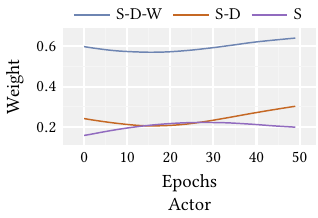}
     \includegraphics[width=0.49\linewidth]{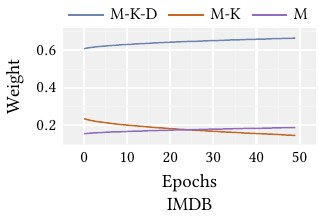}
    \caption{ 
    Visualization of $\{\gamma_l\}$ in Adaptive Heterogeneous Convolution}
    \label{fig:component-visualization}
\end{figure}

\begin{figure}[htbp]
    \centering
    \includegraphics[width=\linewidth]{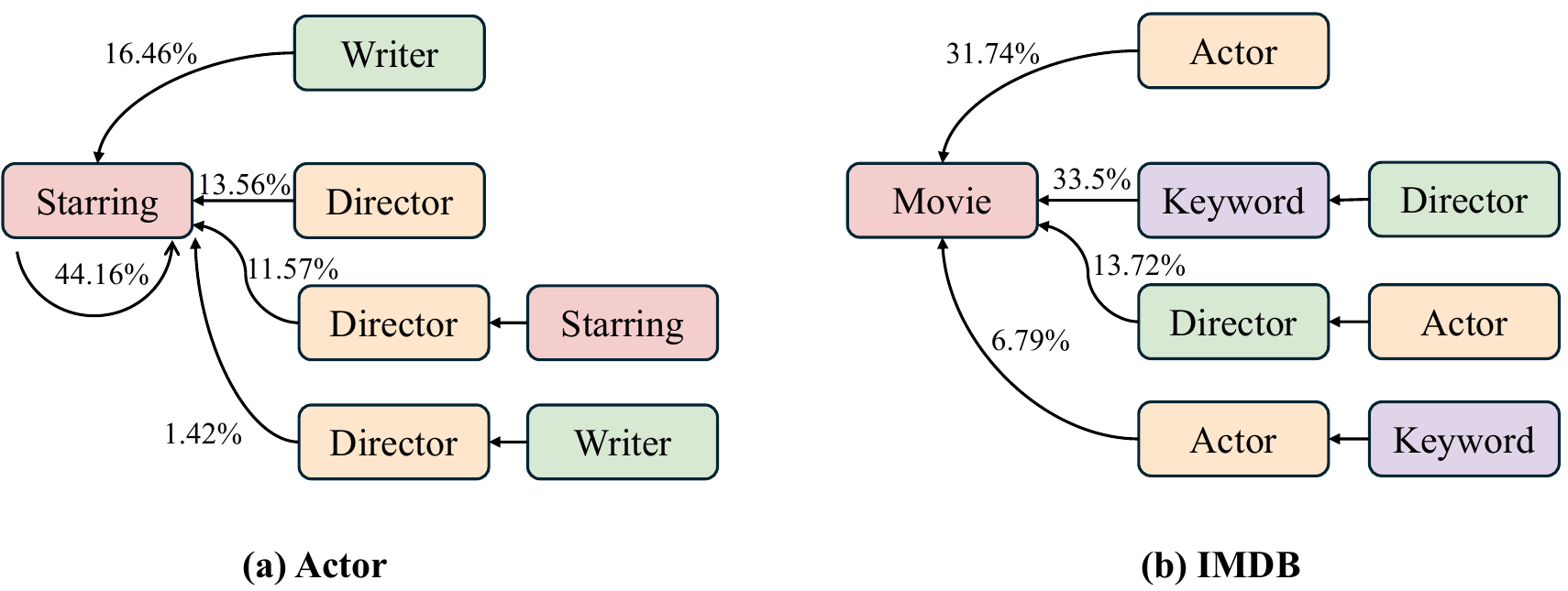}
    \caption{Visualization of meta-paths for representing the target nodes in Actor and IMDB.}
    \label{fig:mp-weights}
\end{figure}

In this subsection, we first present a visualization of the learnable parameters ${\gamma_k}$ associated with each meta-path in the Adaptive Heterogeneous Convolution module (see \autoref{fig:component-visualization}). In \autoref{fig:mp-weights}, we observe that:
(i) The parameters $\{\gamma_l\}$ play their role in modulating the importance of different hops. By adaptively adjusting hop weights during training, the model learns more informative and representative embeddings for each meta-path.
(ii) The evolution of $\{\gamma_l\}$ varies across different hops and graphs, highlighting the need for adaptive modulation. This observation supports our argument in \autoref{sec:intro} that a one-size-fits-all approach is suboptimal.

Next, we visualize the impact of individual meta-paths by examining the scaling factor $\boldsymbol{\beta}$ in the Coarse-to-Fine Semantic Fusion module. The results show that meta-paths with distinct semantic contributions are assigned varying levels of influence. For instance, highly informative meta-paths such as \textit{Starring-Starring} receive greater emphasis, whereas less relevant ones, such as \textit{Starring-Director-Writer}, are assigned lower importance.
\section{Related Works}
\label{sec:related_works}

\subsection{Heterogeneous Graph Neural Networks}

\paragraph{Meta-path-based HGNNs} Meta-path-based approaches utilize predefined or learned meta-paths for feature propagation and semantic fusion. For instance, HAN \cite{wang2019heterogeneous} and MAGNN \cite{fu2020magnn} incorporate graph attention mechanisms, while SeHGNN \cite{yang2023simple} pre-computes message passing prior to training. Seq-HGNN \cite{du2023seq} leverages sequential node embeddings. Additionally, LMSPS \cite{li2023long} introduces a shrinkable sampling strategy for efficient meta-path searching. Alongside model design, there have been efforts to address heterophily through objective function modifications. Hetero$^2$Net \cite{li2023hetero} proposes a masked meta-path strategy and label prediction tasks, while HGMS \cite{wang2025homophily} and LatGRL \cite{shen2025heterophily} adopt contrastive learning objectives. H$^2$SGNN takes a spectral perspective \cite{lu2024addressing}, while H$^2$Gformer involves graph transformers \cite{lin2024heterophily} to handle heterophily. However, they fail to handle diverse heterophily distributions across varying hops and meta-paths.

\paragraph{Meta-path-free HGNNs} Meta-path-free approaches aggregate neighbor messages similar to homogeneous GNNs, while incorporating additional features such as attention mechanisms or positional encodings for enhanced semantic representation \cite{zhu2019relation, hong2020attention}. HetGNN \cite{zhang2019heterogeneous} uses random walks to sample semantically consistent neighbors, while SHGN \cite{lv2021we} combines node features with learnable edge embeddings for heterogeneous attention. HGT \cite{hu2020heterogeneous} and HINormer \cite{mao2023hinormer} further introduce Transformer-style heterogeneous mutual attention. Additionally, LSPI \cite{zhao2025lspi} proposes dividing meta-paths into large and small neighbor paths for improved representation.

\subsection{Heterophily Graph Learning}
Traditional GNNs assume homophily (similarity between connected nodes) \cite{mcpherson2001birds,luan2024heterophilic} and are less effective on heterophily graphs where connected nodes differ significantly. Recent works have extended GNNs to heterophily graphs, focusing primarily on homogeneous graphs. Models like H2GCN \cite{zhu2020beyond} and GPR-GNN \cite{chien2020adaptive} enhance message passing with high-order re-weighting techniques for better heterophily handling. Other approaches, such as LINKX \cite{lim2021large}, GloGNN \cite{li2022finding}, MWGNN \cite{ma2022meta}, ACMGCN \cite{luan2022revisiting}, ALT-GNN \cite{xu2023node}, and more \cite{ma2023rethinking}, refine graph convolution for heterophily, excelling on homogeneous data but struggling with heterogeneous graphs due to lack of type-awareness. HDHGR \cite{guo2023homophily} adapts GNNs to heterophily HGs via graph rewiring but underperforms compared to other HGNNs, as shown in \autoref{table:main-node-classification}. The key difference lies in handling heterophily: HDHGR rewires the graph, while AHGNN directly models varying heterophily and semantic spaces.

\section{Conclusion}
In this paper, we aim to improve the performance of Heterogeneous Graph Neural Networks (HGNNs) on heterophilous data. We identify two key challenges:  
(\romannumeral 1) the variation of heterophily distributions across hops and meta-paths, and  
(\romannumeral 2) the complex, heterophily-influenced semantic variation among different meta-paths.
To address these issues, we propose the \textbf{Adaptive Heterogeneous Graph Neural Network (AHGNN)}. AHGNN conducts hop-specific and meta-path-specific graph convolution in a heterophily-aware fashion. It further refines node representations through a \textbf{Coarse-to-Fine Semantic Fusion} mechanism, which amplifies informative semantics while suppressing noisy signals.
Extensive experiments on seven real-world datasets demonstrate that AHGNN achieves superior performance and efficiency, particularly on graphs characterized by strong heterophily.

\begin{acks}
This work was supported by the National Natural Science Foundation of China (Grant No. 62276006)
\end{acks}

\section{Declaration on Generative AI Tools}

The generative AI tools employed in this work are solely used for language polishing and refinement. These tools do not contribute to the development of the core ideas, methodologies, or results presented in the work. 

\bibliographystyle{ACM-Reference-Format}
\bibliography{acmart}


\begin{thebibliography}{41}


\ifx \showCODEN    \undefined \def \showCODEN     #1{\unskip}     \fi
\ifx \showISBNx    \undefined \def \showISBNx     #1{\unskip}     \fi
\ifx \showISBNxiii \undefined \def \showISBNxiii  #1{\unskip}     \fi
\ifx \showISSN     \undefined \def \showISSN      #1{\unskip}     \fi
\ifx \showLCCN     \undefined \def \showLCCN      #1{\unskip}     \fi
\ifx \shownote     \undefined \def \shownote      #1{#1}          \fi
\ifx \showarticletitle \undefined \def \showarticletitle #1{#1}   \fi
\ifx \showURL      \undefined \def \showURL       {\relax}        \fi
\providecommand\bibfield[2]{#2}
\providecommand\bibinfo[2]{#2}
\providecommand\natexlab[1]{#1}
\providecommand\showeprint[2][]{arXiv:#2}

\bibitem[Bandyopadhyay et~al\mbox{.}(2005)]%
        {bandyopadhyay2005link}
\bibfield{author}{\bibinfo{person}{Sanghamitra Bandyopadhyay},
  \bibinfo{person}{Ujjwal Maulik}, \bibinfo{person}{Lawrence~B Holder},
  \bibinfo{person}{Diane~J Cook}, {and} \bibinfo{person}{Lise Getoor}.}
  \bibinfo{year}{2005}\natexlab{}.
\newblock \showarticletitle{Link-based classification}.
\newblock \bibinfo{journal}{\emph{Advanced methods for knowledge discovery from
  complex data}} (\bibinfo{year}{2005}), \bibinfo{pages}{189--207}.
\newblock


\bibitem[Chen et~al\mbox{.}(2025)]%
        {chen2025dagprompt}
\bibfield{author}{\bibinfo{person}{Qin Chen}, \bibinfo{person}{Liang Wang},
  \bibinfo{person}{Bo Zheng}, {and} \bibinfo{person}{Guojie Song}.}
  \bibinfo{year}{2025}\natexlab{}.
\newblock \showarticletitle{Dagprompt: Pushing the limits of graph prompting
  with a distribution-aware graph prompt tuning approach}. In
  \bibinfo{booktitle}{\emph{Proceedings of the ACM on Web Conference 2025}}.
  \bibinfo{pages}{4346--4358}.
\newblock


\bibitem[Chien et~al\mbox{.}(2021)]%
        {chien2020adaptive}
\bibfield{author}{\bibinfo{person}{Eli Chien}, \bibinfo{person}{Jianhao Peng},
  \bibinfo{person}{Pan Li}, {and} \bibinfo{person}{Olgica Milenkovic}.}
  \bibinfo{year}{2021}\natexlab{}.
\newblock \showarticletitle{Adaptive universal generalized pagerank graph
  neural network}.
\newblock \bibinfo{journal}{\emph{9th International Conference on Learning
  Representations, {ICLR} 2021, Virtual Event, Austria, May 3-7, 2021}}
  (\bibinfo{year}{2021}).
\newblock


\bibitem[Du et~al\mbox{.}(2023)]%
        {du2023seq}
\bibfield{author}{\bibinfo{person}{Chenguang Du}, \bibinfo{person}{Kaichun
  Yao}, \bibinfo{person}{Hengshu Zhu}, \bibinfo{person}{Deqing Wang},
  \bibinfo{person}{Fuzhen Zhuang}, {and} \bibinfo{person}{Hui Xiong}.}
  \bibinfo{year}{2023}\natexlab{}.
\newblock \showarticletitle{Seq-HGNN: Learning Sequential Node Representation
  on Heterogeneous Graph}.
\newblock \bibinfo{journal}{\emph{Proceedings of the 46th International {ACM}
  {SIGIR} Conference on Research and Development in Information Retrieval,
  {SIGIR} 2023, Taipei, Taiwan, July 23-27, 2023}} (\bibinfo{year}{2023}).
\newblock


\bibitem[Fu et~al\mbox{.}(2020)]%
        {fu2020magnn}
\bibfield{author}{\bibinfo{person}{Xinyu Fu}, \bibinfo{person}{Jiani Zhang},
  \bibinfo{person}{Ziqiao Meng}, {and} \bibinfo{person}{Irwin King}.}
  \bibinfo{year}{2020}\natexlab{}.
\newblock \showarticletitle{Magnn: Metapath aggregated graph neural network for
  heterogeneous graph embedding}. In \bibinfo{booktitle}{\emph{Proceedings of
  The Web Conference 2020}}. \bibinfo{pages}{2331--2341}.
\newblock


\bibitem[Gasteiger et~al\mbox{.}(2018)]%
        {gasteiger2018predict}
\bibfield{author}{\bibinfo{person}{Johannes Gasteiger},
  \bibinfo{person}{Aleksandar Bojchevski}, {and} \bibinfo{person}{Stephan
  G{\"u}nnemann}.} \bibinfo{year}{2018}\natexlab{}.
\newblock \showarticletitle{Predict then propagate: Graph neural networks meet
  personalized pagerank}.
\newblock \bibinfo{journal}{\emph{arXiv preprint arXiv:1810.05997}}
  (\bibinfo{year}{2018}).
\newblock


\bibitem[Gong et~al\mbox{.}(2024)]%
        {gong2024towards}
\bibfield{author}{\bibinfo{person}{Chenghua Gong}, \bibinfo{person}{Yao Cheng},
  \bibinfo{person}{Xiang Li}, \bibinfo{person}{Caihua Shan},
  \bibinfo{person}{Siqiang Luo}, {and} \bibinfo{person}{Chuan Shi}.}
  \bibinfo{year}{2024}\natexlab{}.
\newblock \showarticletitle{Towards learning from graphs with heterophily:
  Progress and future}.
\newblock \bibinfo{journal}{\emph{arXiv preprint arXiv:2401.09769}}
  (\bibinfo{year}{2024}).
\newblock


\bibitem[Guo et~al\mbox{.}(2023)]%
        {guo2023homophily}
\bibfield{author}{\bibinfo{person}{Jiayan Guo}, \bibinfo{person}{Lun Du},
  \bibinfo{person}{Wendong Bi}, \bibinfo{person}{Qiang Fu},
  \bibinfo{person}{Xiaojun Ma}, \bibinfo{person}{Xu Chen}, \bibinfo{person}{Shi
  Han}, \bibinfo{person}{Dongmei Zhang}, {and} \bibinfo{person}{Yan Zhang}.}
  \bibinfo{year}{2023}\natexlab{}.
\newblock \showarticletitle{Homophily-oriented Heterogeneous Graph Rewiring}.
  In \bibinfo{booktitle}{\emph{Proceedings of the ACM Web Conference 2023}}.
  \bibinfo{pages}{511--522}.
\newblock


\bibitem[Hong et~al\mbox{.}(2020)]%
        {hong2020attention}
\bibfield{author}{\bibinfo{person}{Huiting Hong}, \bibinfo{person}{Hantao Guo},
  \bibinfo{person}{Yucheng Lin}, \bibinfo{person}{Xiaoqing Yang},
  \bibinfo{person}{Zang Li}, {and} \bibinfo{person}{Jieping Ye}.}
  \bibinfo{year}{2020}\natexlab{}.
\newblock \showarticletitle{An attention-based graph neural network for
  heterogeneous structural learning}. In \bibinfo{booktitle}{\emph{Proceedings
  of the AAAI conference on artificial intelligence}},
  Vol.~\bibinfo{volume}{34}. \bibinfo{pages}{4132--4139}.
\newblock


\bibitem[Hu et~al\mbox{.}(2021)]%
        {hu2021ogb}
\bibfield{author}{\bibinfo{person}{Weihua Hu}, \bibinfo{person}{Matthias Fey},
  \bibinfo{person}{Hongyu Ren}, \bibinfo{person}{Maho Nakata},
  \bibinfo{person}{Yuxiao Dong}, {and} \bibinfo{person}{Jure Leskovec}.}
  \bibinfo{year}{2021}\natexlab{}.
\newblock \showarticletitle{Ogb-lsc: A large-scale challenge for machine
  learning on graphs}.
\newblock \bibinfo{journal}{\emph{arXiv preprint arXiv:2103.09430}}
  (\bibinfo{year}{2021}).
\newblock


\bibitem[Hu et~al\mbox{.}(2020)]%
        {hu2020heterogeneous}
\bibfield{author}{\bibinfo{person}{Ziniu Hu}, \bibinfo{person}{Yuxiao Dong},
  \bibinfo{person}{Kuansan Wang}, {and} \bibinfo{person}{Yizhou Sun}.}
  \bibinfo{year}{2020}\natexlab{}.
\newblock \showarticletitle{Heterogeneous graph transformer}. In
  \bibinfo{booktitle}{\emph{Proceedings of the web conference 2020}}.
  \bibinfo{pages}{2704--2710}.
\newblock


\bibitem[Kingma and Ba(2014)]%
        {kingma2014adam}
\bibfield{author}{\bibinfo{person}{Diederik~P Kingma} {and}
  \bibinfo{person}{Jimmy Ba}.} \bibinfo{year}{2014}\natexlab{}.
\newblock \showarticletitle{Adam: A method for stochastic optimization}.
\newblock \bibinfo{journal}{\emph{3rd International Conference on Learning
  Representations, {ICLR} 2015, San Diego, CA, USA, May 7-9, 2015, Conference
  Track Proceedings}} (\bibinfo{year}{2014}).
\newblock


\bibitem[Kipf and Welling(2016)]%
        {kipf2016semi}
\bibfield{author}{\bibinfo{person}{Thomas~N Kipf} {and} \bibinfo{person}{Max
  Welling}.} \bibinfo{year}{2016}\natexlab{}.
\newblock \showarticletitle{Semi-supervised classification with graph
  convolutional networks}.
\newblock \bibinfo{journal}{\emph{arXiv preprint arXiv:1609.02907}}
  (\bibinfo{year}{2016}).
\newblock


\bibitem[Li et~al\mbox{.}(2023a)]%
        {li2023long}
\bibfield{author}{\bibinfo{person}{Chao Li}, \bibinfo{person}{Zijie Guo},
  \bibinfo{person}{Qiuting He}, \bibinfo{person}{Hao Xu}, {and}
  \bibinfo{person}{Kun He}.} \bibinfo{year}{2023}\natexlab{a}.
\newblock \showarticletitle{Long-range Dependency based Multi-Layer Perceptron
  for Heterogeneous Information Networks}.
\newblock \bibinfo{journal}{\emph{arXiv preprint arXiv:2307.08430}}
  (\bibinfo{year}{2023}).
\newblock


\bibitem[Li et~al\mbox{.}(2023b)]%
        {li2023hetero}
\bibfield{author}{\bibinfo{person}{Jintang Li}, \bibinfo{person}{Zheng Wei},
  \bibinfo{person}{Jiawang Dan}, \bibinfo{person}{Jing Zhou},
  \bibinfo{person}{Yuchang Zhu}, \bibinfo{person}{Ruofan Wu},
  \bibinfo{person}{Baokun Wang}, \bibinfo{person}{Zhang Zhen},
  \bibinfo{person}{Changhua Meng}, \bibinfo{person}{Hong Jin}, {et~al\mbox{.}}}
  \bibinfo{year}{2023}\natexlab{b}.
\newblock \showarticletitle{Hetero$^{2}$ Net: Heterophily-aware Representation
  Learning on Heterogenerous Graphs}.
\newblock \bibinfo{journal}{\emph{arXiv preprint arXiv:2310.11664}}
  (\bibinfo{year}{2023}).
\newblock


\bibitem[Li et~al\mbox{.}(2022)]%
        {li2022finding}
\bibfield{author}{\bibinfo{person}{Xiang Li}, \bibinfo{person}{Renyu Zhu},
  \bibinfo{person}{Yao Cheng}, \bibinfo{person}{Caihua Shan},
  \bibinfo{person}{Siqiang Luo}, \bibinfo{person}{Dongsheng Li}, {and}
  \bibinfo{person}{Weining Qian}.} \bibinfo{year}{2022}\natexlab{}.
\newblock \showarticletitle{Finding global homophily in graph neural networks
  when meeting heterophily}. In \bibinfo{booktitle}{\emph{International
  Conference on Machine Learning}}. PMLR, \bibinfo{pages}{13242--13256}.
\newblock


\bibitem[Lim et~al\mbox{.}(2021)]%
        {lim2021large}
\bibfield{author}{\bibinfo{person}{Derek Lim}, \bibinfo{person}{Felix Hohne},
  \bibinfo{person}{Xiuyu Li}, \bibinfo{person}{Sijia~Linda Huang},
  \bibinfo{person}{Vaishnavi Gupta}, \bibinfo{person}{Omkar Bhalerao}, {and}
  \bibinfo{person}{Ser~Nam Lim}.} \bibinfo{year}{2021}\natexlab{}.
\newblock \showarticletitle{Large scale learning on non-homophilous graphs: New
  benchmarks and strong simple methods}.
\newblock \bibinfo{journal}{\emph{Advances in Neural Information Processing
  Systems}}  \bibinfo{volume}{34} (\bibinfo{year}{2021}),
  \bibinfo{pages}{20887--20902}.
\newblock


\bibitem[Lin et~al\mbox{.}(2024)]%
        {lin2024heterophily}
\bibfield{author}{\bibinfo{person}{Junhong Lin}, \bibinfo{person}{Xiaojie Guo},
  \bibinfo{person}{Shuaicheng Zhang}, \bibinfo{person}{Dawei Zhou},
  \bibinfo{person}{Yada Zhu}, {and} \bibinfo{person}{Julian Shun}.}
  \bibinfo{year}{2024}\natexlab{}.
\newblock \showarticletitle{When Heterophily Meets Heterogeneity: New Graph
  Benchmarks and Effective Methods}.
\newblock \bibinfo{journal}{\emph{arXiv preprint arXiv:2407.10916}}
  (\bibinfo{year}{2024}).
\newblock


\bibitem[Lu et~al\mbox{.}(2024)]%
        {lu2024addressing}
\bibfield{author}{\bibinfo{person}{Kangkang Lu}, \bibinfo{person}{Yanhua Yu},
  \bibinfo{person}{Zhiyong Huang}, \bibinfo{person}{Jia Li},
  \bibinfo{person}{Yuling Wang}, \bibinfo{person}{Meiyu Liang},
  \bibinfo{person}{Xiting Qin}, \bibinfo{person}{Yimeng Ren},
  \bibinfo{person}{Tat-Seng Chua}, {and} \bibinfo{person}{Xidian Wang}.}
  \bibinfo{year}{2024}\natexlab{}.
\newblock \showarticletitle{Addressing Heterogeneity and Heterophily in Graphs:
  A Heterogeneous Heterophilic Spectral Graph Neural Network}.
\newblock \bibinfo{journal}{\emph{arXiv preprint arXiv:2410.13373}}
  (\bibinfo{year}{2024}).
\newblock


\bibitem[Luan et~al\mbox{.}(2024)]%
        {luan2024heterophilic}
\bibfield{author}{\bibinfo{person}{Sitao Luan}, \bibinfo{person}{Chenqing Hua},
  \bibinfo{person}{Qincheng Lu}, \bibinfo{person}{Liheng Ma},
  \bibinfo{person}{Lirong Wu}, \bibinfo{person}{Xinyu Wang},
  \bibinfo{person}{Minkai Xu}, \bibinfo{person}{Xiao-Wen Chang},
  \bibinfo{person}{Doina Precup}, \bibinfo{person}{Rex Ying}, {et~al\mbox{.}}}
  \bibinfo{year}{2024}\natexlab{}.
\newblock \showarticletitle{The heterophilic graph learning handbook:
  Benchmarks, models, theoretical analysis, applications and challenges}.
\newblock \bibinfo{journal}{\emph{arXiv preprint arXiv:2407.09618}}
  (\bibinfo{year}{2024}).
\newblock


\bibitem[Luan et~al\mbox{.}(2022)]%
        {luan2022revisiting}
\bibfield{author}{\bibinfo{person}{Sitao Luan}, \bibinfo{person}{Chenqing Hua},
  \bibinfo{person}{Qincheng Lu}, \bibinfo{person}{Jiaqi Zhu},
  \bibinfo{person}{Mingde Zhao}, \bibinfo{person}{Shuyuan Zhang},
  \bibinfo{person}{Xiao-Wen Chang}, {and} \bibinfo{person}{Doina Precup}.}
  \bibinfo{year}{2022}\natexlab{}.
\newblock \showarticletitle{Revisiting heterophily for graph neural networks}.
\newblock \bibinfo{journal}{\emph{Advances in neural information processing
  systems}}  \bibinfo{volume}{35} (\bibinfo{year}{2022}),
  \bibinfo{pages}{1362--1375}.
\newblock


\bibitem[Lv et~al\mbox{.}(2021)]%
        {lv2021we}
\bibfield{author}{\bibinfo{person}{Qingsong Lv}, \bibinfo{person}{Ming Ding},
  \bibinfo{person}{Qiang Liu}, \bibinfo{person}{Yuxiang Chen},
  \bibinfo{person}{Wenzheng Feng}, \bibinfo{person}{Siming He},
  \bibinfo{person}{Chang Zhou}, \bibinfo{person}{Jianguo Jiang},
  \bibinfo{person}{Yuxiao Dong}, {and} \bibinfo{person}{Jie Tang}.}
  \bibinfo{year}{2021}\natexlab{}.
\newblock \showarticletitle{Are we really making much progress? revisiting,
  benchmarking and refining heterogeneous graph neural networks}. In
  \bibinfo{booktitle}{\emph{Proceedings of the 27th ACM SIGKDD conference on
  knowledge discovery \& data mining}}. \bibinfo{pages}{1150--1160}.
\newblock


\bibitem[Ma et~al\mbox{.}(2022)]%
        {ma2022meta}
\bibfield{author}{\bibinfo{person}{Xiaojun Ma}, \bibinfo{person}{Qin Chen},
  \bibinfo{person}{Yuanyi Ren}, \bibinfo{person}{Guojie Song}, {and}
  \bibinfo{person}{Liang Wang}.} \bibinfo{year}{2022}\natexlab{}.
\newblock \showarticletitle{Meta-weight graph neural network: Push the limits
  beyond global homophily}. In \bibinfo{booktitle}{\emph{Proceedings of the ACM
  Web Conference 2022}}. \bibinfo{pages}{1270--1280}.
\newblock


\bibitem[Ma et~al\mbox{.}(2023)]%
        {ma2023rethinking}
\bibfield{author}{\bibinfo{person}{Xiaojun Ma}, \bibinfo{person}{Qin Chen},
  \bibinfo{person}{Yi Wu}, \bibinfo{person}{Guojie Song},
  \bibinfo{person}{Liang Wang}, {and} \bibinfo{person}{Bo Zheng}.}
  \bibinfo{year}{2023}\natexlab{}.
\newblock \showarticletitle{Rethinking structural encodings: Adaptive graph
  transformer for node classification task}. In
  \bibinfo{booktitle}{\emph{Proceedings of the ACM web conference 2023}}.
  \bibinfo{pages}{533--544}.
\newblock


\bibitem[Mao et~al\mbox{.}(2023)]%
        {mao2023hinormer}
\bibfield{author}{\bibinfo{person}{Qiheng Mao}, \bibinfo{person}{Zemin Liu},
  \bibinfo{person}{Chenghao Liu}, {and} \bibinfo{person}{Jianling Sun}.}
  \bibinfo{year}{2023}\natexlab{}.
\newblock \showarticletitle{Hinormer: Representation learning on heterogeneous
  information networks with graph transformer}. In
  \bibinfo{booktitle}{\emph{Proceedings of the ACM Web Conference 2023}}.
  \bibinfo{pages}{599--610}.
\newblock


\bibitem[McPherson et~al\mbox{.}(2001)]%
        {mcpherson2001birds}
\bibfield{author}{\bibinfo{person}{Miller McPherson}, \bibinfo{person}{Lynn
  Smith-Lovin}, {and} \bibinfo{person}{James~M Cook}.}
  \bibinfo{year}{2001}\natexlab{}.
\newblock \showarticletitle{Birds of a feather: Homophily in social networks}.
\newblock \bibinfo{journal}{\emph{Annual review of sociology}}
  \bibinfo{volume}{27}, \bibinfo{number}{1} (\bibinfo{year}{2001}),
  \bibinfo{pages}{415--444}.
\newblock


\bibitem[Pillai et~al\mbox{.}(2005)]%
        {pillai2005perron}
\bibfield{author}{\bibinfo{person}{S~Unnikrishna Pillai},
  \bibinfo{person}{Torsten Suel}, {and} \bibinfo{person}{Seunghun Cha}.}
  \bibinfo{year}{2005}\natexlab{}.
\newblock \showarticletitle{The Perron-Frobenius theorem: some of its
  applications}.
\newblock \bibinfo{journal}{\emph{IEEE Signal Processing Magazine}}
  \bibinfo{volume}{22}, \bibinfo{number}{2} (\bibinfo{year}{2005}),
  \bibinfo{pages}{62--75}.
\newblock


\bibitem[Shen and Kang(2025)]%
        {shen2025heterophily}
\bibfield{author}{\bibinfo{person}{Zhixiang Shen} {and} \bibinfo{person}{Zhao
  Kang}.} \bibinfo{year}{2025}\natexlab{}.
\newblock \showarticletitle{When heterophily meets heterogeneous graphs: Latent
  graphs guided unsupervised representation learning}.
\newblock \bibinfo{journal}{\emph{IEEE Transactions on Neural Networks and
  Learning Systems}} (\bibinfo{year}{2025}).
\newblock


\bibitem[Sun and Han(2012)]%
        {sun2012mining}
\bibfield{author}{\bibinfo{person}{Yizhou Sun} {and} \bibinfo{person}{Jiawei
  Han}.} \bibinfo{year}{2012}\natexlab{}.
\newblock \bibinfo{booktitle}{\emph{Mining heterogeneous information networks:
  principles and methodologies}}.
\newblock \bibinfo{publisher}{Morgan \& Claypool Publishers}.
\newblock


\bibitem[Tang et~al\mbox{.}(2009)]%
        {tang2009social}
\bibfield{author}{\bibinfo{person}{Jie Tang}, \bibinfo{person}{Jimeng Sun},
  \bibinfo{person}{Chi Wang}, {and} \bibinfo{person}{Zi Yang}.}
  \bibinfo{year}{2009}\natexlab{}.
\newblock \showarticletitle{Social influence analysis in large-scale networks}.
  In \bibinfo{booktitle}{\emph{Proceedings of the 15th ACM SIGKDD international
  conference on Knowledge discovery and data mining}}.
  \bibinfo{pages}{807--816}.
\newblock


\bibitem[Traud et~al\mbox{.}(2012)]%
        {traud2012social}
\bibfield{author}{\bibinfo{person}{Amanda~L Traud}, \bibinfo{person}{Peter~J
  Mucha}, {and} \bibinfo{person}{Mason~A Porter}.}
  \bibinfo{year}{2012}\natexlab{}.
\newblock \showarticletitle{Social structure of facebook networks}.
\newblock \bibinfo{journal}{\emph{Physica A: Statistical Mechanics and its
  Applications}} \bibinfo{volume}{391}, \bibinfo{number}{16}
  (\bibinfo{year}{2012}), \bibinfo{pages}{4165--4180}.
\newblock


\bibitem[Veli{\v{c}}kovi{\'c} et~al\mbox{.}(2017)]%
        {velivckovic2017graph}
\bibfield{author}{\bibinfo{person}{Petar Veli{\v{c}}kovi{\'c}},
  \bibinfo{person}{Guillem Cucurull}, \bibinfo{person}{Arantxa Casanova},
  \bibinfo{person}{Adriana Romero}, \bibinfo{person}{Pietro Lio}, {and}
  \bibinfo{person}{Yoshua Bengio}.} \bibinfo{year}{2017}\natexlab{}.
\newblock \showarticletitle{Graph attention networks}.
\newblock \bibinfo{journal}{\emph{arXiv preprint arXiv:1710.10903}}
  (\bibinfo{year}{2017}).
\newblock


\bibitem[Wang et~al\mbox{.}(2025)]%
        {wang2025homophily}
\bibfield{author}{\bibinfo{person}{Haosen Wang}, \bibinfo{person}{Chenglong
  Shi}, \bibinfo{person}{Can Xu}, \bibinfo{person}{Surong Yan}, {and}
  \bibinfo{person}{Pan Tang}.} \bibinfo{year}{2025}\natexlab{}.
\newblock \showarticletitle{Homophily-aware Heterogeneous Graph Contrastive
  Learning}.
\newblock \bibinfo{journal}{\emph{arXiv preprint arXiv:2501.08538}}
  (\bibinfo{year}{2025}).
\newblock


\bibitem[Wang et~al\mbox{.}(2019)]%
        {wang2019heterogeneous}
\bibfield{author}{\bibinfo{person}{Xiao Wang}, \bibinfo{person}{Houye Ji},
  \bibinfo{person}{Chuan Shi}, \bibinfo{person}{Bai Wang},
  \bibinfo{person}{Yanfang Ye}, \bibinfo{person}{Peng Cui}, {and}
  \bibinfo{person}{Philip~S Yu}.} \bibinfo{year}{2019}\natexlab{}.
\newblock \showarticletitle{Heterogeneous graph attention network}. In
  \bibinfo{booktitle}{\emph{The world wide web conference}}.
  \bibinfo{pages}{2022--2032}.
\newblock


\bibitem[Wu et~al\mbox{.}(2019)]%
        {wu2019simplifying}
\bibfield{author}{\bibinfo{person}{Felix Wu}, \bibinfo{person}{Amauri Souza},
  \bibinfo{person}{Tianyi Zhang}, \bibinfo{person}{Christopher Fifty},
  \bibinfo{person}{Tao Yu}, {and} \bibinfo{person}{Kilian Weinberger}.}
  \bibinfo{year}{2019}\natexlab{}.
\newblock \showarticletitle{Simplifying graph convolutional networks}. In
  \bibinfo{booktitle}{\emph{International conference on machine learning}}.
  PMLR, \bibinfo{pages}{6861--6871}.
\newblock


\bibitem[Xu et~al\mbox{.}(2023)]%
        {xu2023node}
\bibfield{author}{\bibinfo{person}{Zhe Xu}, \bibinfo{person}{Yuzhong Chen},
  \bibinfo{person}{Qinghai Zhou}, \bibinfo{person}{Yuhang Wu},
  \bibinfo{person}{Menghai Pan}, \bibinfo{person}{Hao Yang}, {and}
  \bibinfo{person}{Hanghang Tong}.} \bibinfo{year}{2023}\natexlab{}.
\newblock \showarticletitle{Node classification beyond homophily: Towards a
  general solution}. In \bibinfo{booktitle}{\emph{Proceedings of the 29th ACM
  SIGKDD Conference on Knowledge Discovery and Data Mining}}.
  \bibinfo{pages}{2862--2873}.
\newblock


\bibitem[Yang et~al\mbox{.}(2023)]%
        {yang2023simple}
\bibfield{author}{\bibinfo{person}{Xiaocheng Yang}, \bibinfo{person}{Mingyu
  Yan}, \bibinfo{person}{Shirui Pan}, \bibinfo{person}{Xiaochun Ye}, {and}
  \bibinfo{person}{Dongrui Fan}.} \bibinfo{year}{2023}\natexlab{}.
\newblock \showarticletitle{Simple and efficient heterogeneous graph neural
  network}. In \bibinfo{booktitle}{\emph{Proceedings of the AAAI Conference on
  Artificial Intelligence}}, Vol.~\bibinfo{volume}{37}.
  \bibinfo{pages}{10816--10824}.
\newblock


\bibitem[Zhang et~al\mbox{.}(2019)]%
        {zhang2019heterogeneous}
\bibfield{author}{\bibinfo{person}{Chuxu Zhang}, \bibinfo{person}{Dongjin
  Song}, \bibinfo{person}{Chao Huang}, \bibinfo{person}{Ananthram Swami}, {and}
  \bibinfo{person}{Nitesh~V Chawla}.} \bibinfo{year}{2019}\natexlab{}.
\newblock \showarticletitle{Heterogeneous graph neural network}. In
  \bibinfo{booktitle}{\emph{Proceedings of the 25th ACM SIGKDD international
  conference on knowledge discovery \& data mining}}.
  \bibinfo{pages}{793--803}.
\newblock


\bibitem[Zhao et~al\mbox{.}(2025)]%
        {zhao2025lspi}
\bibfield{author}{\bibinfo{person}{Yufei Zhao}, \bibinfo{person}{Shiduo Wang},
  {and} \bibinfo{person}{Hua Duan}.} \bibinfo{year}{2025}\natexlab{}.
\newblock \showarticletitle{LSPI: Heterogeneous graph neural network
  classification aggregation algorithm based on size neighbor path
  identification}.
\newblock \bibinfo{journal}{\emph{Applied Soft Computing}}
  (\bibinfo{year}{2025}), \bibinfo{pages}{112656}.
\newblock


\bibitem[Zhu et~al\mbox{.}(2020)]%
        {zhu2020beyond}
\bibfield{author}{\bibinfo{person}{Jiong Zhu}, \bibinfo{person}{Yujun Yan},
  \bibinfo{person}{Lingxiao Zhao}, \bibinfo{person}{Mark Heimann},
  \bibinfo{person}{Leman Akoglu}, {and} \bibinfo{person}{Danai Koutra}.}
  \bibinfo{year}{2020}\natexlab{}.
\newblock \showarticletitle{Beyond homophily in graph neural networks: Current
  limitations and effective designs}.
\newblock \bibinfo{journal}{\emph{Advances in neural information processing
  systems}}  \bibinfo{volume}{33} (\bibinfo{year}{2020}),
  \bibinfo{pages}{7793--7804}.
\newblock


\bibitem[Zhu et~al\mbox{.}(2019)]%
        {zhu2019relation}
\bibfield{author}{\bibinfo{person}{Shichao Zhu}, \bibinfo{person}{Chuan Zhou},
  \bibinfo{person}{Shirui Pan}, \bibinfo{person}{Xingquan Zhu}, {and}
  \bibinfo{person}{Bin Wang}.} \bibinfo{year}{2019}\natexlab{}.
\newblock \showarticletitle{Relation structure-aware heterogeneous graph neural
  network}. In \bibinfo{booktitle}{\emph{2019 IEEE international conference on
  data mining (ICDM)}}. IEEE, \bibinfo{pages}{1534--1539}.
\newblock


\end{thebibliography}

\appendix

\end{document}